\documentclass{article}

\PassOptionsToPackage{numbers,sort&compress}{natbib}
\usepackage[preprint]{neurips_2026}

\usepackage[utf8]{inputenc} % allow utf-8 input
\usepackage[T1]{fontenc}    % use 8-bit T1 fonts
\usepackage{hyperref}       % hyperlinks
\usepackage{url}            % simple URL typesetting
\usepackage{booktabs}       % professional-quality tables
\usepackage{amsfonts}       % blackboard math symbols
\usepackage{nicefrac}       % compact symbols for 1/2, etc.
\usepackage{microtype}      % microtypography
\usepackage{xcolor}         % colors
\usepackage{amsmath}
\usepackage{graphicx}
\usepackage{amsthm}
\newtheorem{proposition}{Proposition}
\usepackage{subcaption}

\title{Functional Graphs for Predicting and Explaining Goal Failure in Sparse Goal-Conditioned RL}

\author{
  Shalley Dash\thanks{Corresponding author.} \\
  Associate Adjunct Professor, Business Analytics,\\
  Institute of Management Technology\\
  Ghaziabad, 201001  \\
  \texttt{sdash@imt.edu} \\
}

\begin{document}

\maketitle

\title{Policy-Induced Functional Graphs Explain Success and Failure in Sparse Goal-Conditioned RL}

\begin{abstract}
	Sparse goal-conditioned reinforcement learning can produce policies whose failures are hidden by aggregate success rates. We analyze trained goal-conditioned value policies through the deterministic functional graphs induced by greedy evaluation: for each goal, every state maps to a single successor, decomposing behavior into attractors and basins. This reveals a local-to-global structure in learned policies. We define local goal support (LGS), a one-step statistic measuring the fraction of valid neighboring states whose greedy successor is the goal. In deterministic sparse GridWorlds, zero LGS exactly precludes goal entry from non-goal starts. Empirically, weak LGS is a strong diagnostic of goal-level failure across update rules, curricula, larger grids, and bottleneck geometries: the fixed rule \(\mathrm{LGS}\leq 0.5\) identifies low-success goals with precision \(0.921\), recall \(0.929\), and F1 \(0.925\) in the main \(8{\times}8\) TD setting, with similar performance across variants. However, local support is not sufficient for global success: some supported goals still fail because distant states are captured by competing attractors or fragmented basin structure. We therefore introduce a compact post-hoc taxonomy of policy-induced graphs---goal-dominant, competitor-dominated, partial/contested, and fragmented---to characterize residual failure modes beyond local support. These results show that sparse GCRL failures can be understood as structured policy-induced dynamics, and that local one-step policy structure provides a cheap post-training diagnostic for goal-level failure.
\end{abstract}

\section{Introduction}

Goal-conditioned reinforcement learning (GCRL) aims to learn policies that can reach many goals from many states, often through goal-conditioned value functions such as Universal Value Function Approximators (UVFAs)~\citep{schaul2015uvfa}. In sparse-reward settings, however, learning remains difficult because successful trajectories are rare and credit assignment is weak. Classical responses to this problem include hindsight relabeling and goal curricula, which aim to increase useful goal-reaching experience during training~\citep{andrychowicz2017hindsight,florensa2018automatic, portelas2020automatic}. Yet after training, standard evaluation usually reports aggregate quantities such as average return or success rate. These metrics tell us whether a policy succeeds on average, but provide limited insight into why it succeeds for some goals, fails for others, or varies substantially across random seeds.

This limitation appears even in the controlled setting of sparse GridWorld. Table~\ref{tab:main_performance} shows the motivating empirical pattern in the open \(8{\times}8\) environment. TD learning improves mean greedy-evaluation success relative to Monte Carlo training, but it does not produce uniformly reliable policies: under TD-Uniform, mean evaluation success is \(0.44\), while individual seeds range from \(0.26\) to \(0.67\). An edge-biased goal-sampling curriculum also fails to resolve the issue. Although it changes the training distribution by emphasizing harder edge goals, its mean evaluation success remains close to TD-Uniform and its seed-level variability remains similarly large. Thus, even when the environment, horizon, architecture, update rule, and sampling scheme are fixed, learned policies can organize the same state space in sharply different ways. Aggregate success rates reveal that variability exists, but not what structural form it takes. This motivates our central question: for each goal, what closed-loop dynamics has the learned greedy policy induced, and how do those dynamics determine success or failure?

We study this question by analyzing the deterministic greedy policy induced after training. For a fixed goal \(g\), the learned value function defines a greedy successor map \(f_g:\mathcal{S}\to\mathcal{S}\), mapping each state to the next state selected by the greedy policy. Because each state has exactly one successor, \(f_g\) induces a finite functional graph. Equivalently, it defines a discrete dynamical system over the state space: trajectories eventually enter attractors, and each attractor has an associated basin of states that flow into it. In this representation, successful behavior corresponds to a goal-dominant basin, while failures appear as spurious fixed points, cycles, competing basins, or fragmented collections of attractors.

The key local question is whether the learned dynamics provide immediate inward attraction into the intended goal. In a local-movement GridWorld, every successful trajectory that reaches \(g\) from another state must enter \(g\) through one of its valid one-step neighbors. We therefore define \emph{local goal support} (LGS) as the fraction of valid neighboring states whose greedy successor is the goal. This is a simple local stability diagnostic: it measures whether the learned policy points into the goal in its immediate neighborhood. If \(\mathrm{LGS}(g)=0\), then no non-goal trajectory can enter \(g\), giving a local obstruction to finite-horizon success. Empirically, this diagnostic strongly predicts goal-level failure across TD, Monte Carlo, curriculum, larger-grid, and bottleneck variants.

However, local attraction is not sufficient for global success. A goal may be locally supported while distant states are captured by a non-goal attractor or split across several basins. This creates an asymmetry: weak local support can certify or strongly indicate failure, but success depends on the global basin structure of the full policy-induced graph. We therefore combine local diagnostics with global functional-graph quantities, including goal-basin fraction, largest competing basin, and fragmentation.

To the best of our knowledge, policy-induced functional-graph decomposition has not previously been used as an exact post-training diagnostic for learned sparse GCRL policies. Our contribution is not a new training algorithm, but a structural evaluation framework that connects local goal attraction, global basin organization, and observed success. Specifically, we make three contributions:
\begin{enumerate}
	\item We formalize the greedy evaluation policy of a trained GCRL agent as a goal-indexed functional graph, equivalently a finite discrete dynamical system, and use its attractors and basins to analyze policy behavior across goals.
	\item We introduce local goal support as an operational measure of local inward attraction around a goal and prove that zero support is a necessary obstruction to goal reachability in deterministic local-movement GridWorlds.
	\item Across TD, Monte Carlo, curriculum, \(12{\times}12\), and bottleneck variants, we show that local support predicts goal-level failure, while global graph structure separates residual failures into competitor-dominated, fragmented, partial, and goal-dominant regimes.
\end{enumerate}

\section{Related Work}
\label{sec:related_work}

\paragraph{Goal-conditioned RL and goal-space structure.}
Goal-conditioned reinforcement learning studies policies and value functions conditioned on desired goals, with UVFAs providing an early framework for value generalization across goals~\citep{schaul2015uvfa}. Sparse goal-reaching has traditionally motivated hindsight relabeling and automatic curricula~\citep{andrychowicz2017hindsight,florensa2018automatic,portelas2020automatic}. More recent work has broadened this toolkit through supervised goal-conditioned learning~\citep{ghosh2021learning,yang2022rethinking}, contrastive goal-conditioned representations~\citep{eysenbach2022contrastive}, and temporal-distance or quasimetric value representations~\citep{bae2024tldr,myers2025offline}. These approaches make sparse goal-reaching more learnable by using supervised losses, contrastive objectives, or distance-like structure to improve learning and generalization. Our focus is different: after a goal-conditioned policy has been trained, we ask what closed-loop structure it induces over the state space and how that structure explains goal-wise success and failure.

\paragraph{Geometric, dynamical, and structural priors in RL.}

More broadly, there is growing interest in incorporating geometric, physical, or dynamical structure into RL, including successor-style representations~\citep{dayan1993successor,barreto2017successor} and physics-informed value regularization~\citep{giammarino2025physics}. These works typically use structure as an inductive bias for learning, planning, or transfer. In contrast, we use finite dynamical structure as a post-training diagnostic of the closed-loop greedy policy. Our use of dynamical-systems structure is different. We do not impose a prior during training. Instead, we analyze the deterministic greedy policy after training as a finite discrete dynamical system. For each goal, the policy defines a successor map \(f_g:\mathcal{S}\to\mathcal{S}\), whose attractors and basins reveal how the learned policy organizes behavior. Thus, the dynamical-systems object in our paper is not the environment model or a learned latent transition operator, but the closed-loop policy-induced state map.

\paragraph{Graphs, neural dynamics, and policy diagnostics.}
Graphs are widely used in RL as relational representations, graph-structured policy architectures, planning substrates, and learned abstractions~\citep{zambaldi2019relational,wang2018nervenet,fathinezhad2023gnnrl}. Dynamical-systems ideas have also shaped neural-network architectures, including continuous-depth models, equilibrium networks, and graph neural dynamics~\citep{chen2018neural,bai2019deep,xhonneux2020continuous}. These works typically use graphs or dynamics as representational, architectural, or computational structure. Our use is different: we combine a graph-theoretic functional-map view with finite discrete-dynamical-system concepts to diagnose the closed-loop behavior induced after training by a greedy goal-conditioned policy. The resulting policy-induced graph is not part of the environment, architecture, or planner; it is an extracted object used to identify missing local goal entry, dominant competing attractors, and fragmentation across basins.

This complements standard aggregate evaluation by turning a trained policy into a structural readout: local entry into the goal, global basin dominance, competing attractors, and fragmentation.

\section{Experimental Setup}
\label{sec:setup}

We instantiate sparse GCRL in finite deterministic GridWorlds. Each episode samples a start state \(s_0\) and a goal \(g\), and the agent receives reward only on reaching the goal; otherwise rewards are zero until the finite horizon \(H\). Thus the task has the central difficulty of sparse goal-conditioned learning: most trajectories provide no positive reward signal, and successful behavior must be learned across many start--goal pairs. We use GridWorlds because our analysis requires exact enumeration of the learned closed-loop policy. For each trained seed and goal, we evaluate all start states, trace the greedy successor map, and compute the complete attractor--basin decomposition. The open \(8{\times}8\) GridWorld is the main setting, with robustness checks in a larger \(12{\times}12\) GridWorld and an \(8{\times}8\) bottleneck variant. This controlled setting is a first step: it preserves the sparse GCRL structure while making the induced policy dynamics fully observable and countable. Details are summarized in Table~\ref{tab:settings}. Full hyperparameters and implementation details are provided in Appendix~\ref{app:implementation_details}.
 
\begin{table}[t]
	\centering
	\footnotesize
	\setlength{\tabcolsep}{3.5pt}
	\renewcommand{\arraystretch}{0.90}
	\caption{Main experimental settings. All evaluations use deterministic greedy rollouts over all valid start--goal pairs. TD denotes one-step temporal-difference learning; MC denotes Monte Carlo return targets.}
	\label{tab:settings}
	\begin{tabular}{lccccc}
		\toprule
		Condition & Grid & Obst. & \(H\) & Learning / sampling & Seeds \\
		\midrule
		Open-8 TD uniform     & \(8{\times}8\)   & No  & 16 & TD, uniform       & 20 \\
		Open-8 MC uniform     & \(8{\times}8\)   & No  & 16 & MC, uniform       & 20 \\
		Open-8 TD edge        & \(8{\times}8\)   & No  & 16 & TD, edge-biased   & 20 \\
		Open-12 TD uniform    & \(12{\times}12\) & No  & 24 & TD, uniform       & 20 \\
		Bottleneck TD uniform & \(8{\times}8\)   & Yes & 24 & TD, uniform       & 20 \\
		\bottomrule
	\end{tabular}
\end{table}

Agents are trained with a goal-conditioned value function \(V_\theta(s,g)\) and \(\epsilon\)-greedy exploration. At evaluation, exploration is disabled. For each trained seed and goal \(g\), we roll out the deterministic greedy policy from every valid start state \(s\neq g\), using the same horizon \(H\) as in training. We record finite-horizon success and the induced trajectory. Because both the environment and evaluated policy are deterministic, each fixed goal induces a unique successor map over states; variation across goals and seeds therefore reflects differences in the learned closed-loop policy structure rather than stochastic evaluation noise.

\paragraph{Aggregate performance hides structural variation.}
\label{sec:aggregate_puzzle}

Table~\ref{tab:main_performance} reports standard aggregate performance in the open \(8{\times}8\) GridWorld over 20 seeds. TD improves mean evaluation success relative to MC, but leaves substantial seed-level variation: TD-Uniform ranges from \(0.26\) to \(0.67\) despite identical environment, architecture, learning rule, and goal-sampling distribution. Edge-biased sampling also does not resolve the variation: although it shifts training exposure toward harder edge goals, its mean evaluation success remains close to TD-Uniform and its variability remains high.

These results show that performance is not explained by update rule or goal exposure alone. Aggregate success does not reveal whether failures arise from missing local entry into the goal, dominant non-goal attractors, cycles, or fragmented basins. We therefore analyze the learned greedy policy itself, asking how it organizes trajectories over the state space for each goal.
\begin{table}[t]
	\centering
	\footnotesize
	\setlength{\tabcolsep}{3.5pt}
	\renewcommand{\arraystretch}{0.90}
	\begin{tabular}{lcccccc}
		\toprule
		Method & Train & Last100 & Eval & Std & KL & Range \\
		\midrule
		MC-Uniform    & 0.19 & 0.22 & 0.19 & 0.07 & 0.03 & [0.08, 0.35] \\
		TD-Uniform    & 0.36 & 0.44 & 0.44 & 0.13 & 0.03 & [0.26, 0.67] \\
		TD-EdgeBias   & 0.35 & 0.40 & 0.46 & 0.13 & 0.04 & [0.15, 0.70] \\
		\bottomrule
	\end{tabular}
	\caption{Aggregate performance across 20 seeds in the \(8\times 8\) open GridWorld. TD improves mean success over MC, but substantial seed-level variability remains. Edge-biased sampling does not consistently improve overall success or reduce variability.}
	\label{tab:main_performance}
\end{table}

\section{Analytical Framework: Policy-Induced Functional Graphs}
\label{sec:framework}

To explain goal-level success and failure, we analyze the closed-loop dynamics induced by a learned goal-conditioned policy under greedy evaluation. For each fixed goal \(g\), the learned policy and deterministic environment define a state map: each state is mapped to the unique successor selected by the greedy policy. Representing this map as a functional graph gives a finite discrete dynamical system with one outgoing edge per state. The standard finite-map decomposition into cycles, attractors, and basins then makes policy behavior explicit: successful goals correspond to states flowing into the goal, while failures correspond to non-goal fixed points, cycles, competing basins, or fragmented basin structure~\citep{harary1969graph,flajolet2009analytic,elaydi2005discrete}. Finite-horizon evaluation scores this structure by counting which start states reach the goal within the rollout horizon.
\paragraph{Greedy Evaluation Induces a Goal-Indexed State Map}

Let \(\mathcal{S}\) denote the finite set of valid states, \(\mathcal{A}\) the action set, and \(T:\mathcal{S}\times\mathcal{A}\to\mathcal{S}\) the deterministic transition function. For a fixed goal \(g\), the learned value function \(V_\theta(s,g)\) induces a deterministic greedy policy \(\pi_g\), with ties resolved by a fixed rule. Executing this policy from an initial state \(s_0\) generates the trajectory
\[
s_0,\; s_1,\; s_2,\ldots,
\qquad s_{t+1}=T(s_t,\pi_g(s_t)).
\]
Equivalently, greedy execution is repeated application of the goal-indexed state map
\[
f_g:\mathcal{S}\to\mathcal{S}, 
\qquad 
f_g(s)=T(s,\pi_g(s)).
\]
Thus \(s_t=f_g^t(s_0)\). Because both the evaluated policy and the environment are deterministic, \(f_g\) assigns a unique successor to every valid state. This map is the closed-loop object analyzed below.

\paragraph{The State Map as a Functional Graph}

The goal-indexed map \(f_g\) can be represented as a directed graph
\(G_g=(\mathcal{S},E_g)\), with one node per valid state and one edge from each
state to its greedy successor:
\[
E_g=\{(s,f_g(s)) : s\in\mathcal{S}\}.
\]
Since every state has exactly one outgoing edge, \(G_g\) is a \emph{functional graph}, or mapping digraph. Functional graphs are the standard graph representation of finite maps: each connected component contains exactly one directed cycle, with all remaining states feeding into that cycle through directed in-trees~\citep{harary1969graph,flajolet2009analytic}. We give the short proof in Appendix~\ref{app:finite_maps}.

Viewed as a finite discrete dynamical system, the cycles of \(G_g\) are attractors and the states that flow into them form basins of attraction~\citep{elaydi2005discrete}. Thus, the graph representation is not only a visualization of the greedy policy; it makes the policy-induced attractor--basin structure exactly computable.

\begin{proposition}[Finite-map decomposition; standard]
	\label{prop:finite_map_decomposition}
	Let \(\mathcal{S}\) be finite and let \(f:\mathcal{S}\to\mathcal{S}\) be deterministic. Then every trajectory
	\[
	s, f(s), f^2(s), \ldots
	\]
	eventually enters a directed cycle. Equivalently, the induced functional graph decomposes into components, each containing exactly one directed cycle and the states whose forward paths feed into that cycle.
\end{proposition}

\begin{proof}
	Because \(\mathcal{S}\) is finite, the sequence \(s, f(s), f^2(s), \ldots\) must eventually repeat a state. Determinism then implies that the repeated segment is a directed cycle. Since each state has exactly one outgoing edge, every forward path enters a unique cycle, and the states whose paths enter the same cycle form one component.
\end{proof}

Applied to \(f_g\), the directed cycles are attractors and their predecessor trees are basins of attraction~\citep{elaydi2005discrete}. Thus each learned greedy policy admits an exact goal-indexed attractor--basin decomposition. The decomposition is classical; our contribution is to use it as a diagnostic for learned sparse GCRL policies, where success corresponds to goal-dominant basins and failures appear as competing non-goal attractors, cycles, or fragmented basin structure.

\subsection{Attractors, Basins, and Failure Modes}

For a fixed goal \(g\), the functional graph of \(f_g\) decomposes the state space into attractors and their basins. We treat the goal as absorbing for graph construction, so \(f_g(g)=g\). Under this convention, the desired attractor is the goal self-loop; all other attractors correspond to failure modes under greedy execution.

Let \(\mathcal{A}_g\) denote the set of attractors of \(f_g\). An attractor may be:
\begin{itemize}
	\item \textbf{Goal attractor:} the absorbing goal state \(\{g\}\), represented as a self-loop.
	\item \textbf{Spurious fixed point:} a non-goal state \(\{s^\star\}\), \(s^\star\neq g\), such that \(f_g(s^\star)=s^\star\).
	\item \textbf{Cycle:} a recurrent orbit \(\{s_1,\ldots,s_k\}\), \(k>1\), in which the policy loops without reaching the goal.
\end{itemize}

The basin of an attractor \(A\in\mathcal{A}_g\) is
\[
\mathcal{B}_g(A)
=
\{s\in\mathcal{S}: \exists t\geq 0 \text{ such that } f_g^t(s)\in A\}.
\]
Because \(f_g\) is deterministic, each state has a unique forward trajectory and therefore belongs to exactly one attractor basin. The goal basin \(\mathcal{B}_g(\{g\})\) contains states whose trajectories eventually reach the goal. Non-goal basins contain states captured by spurious fixed points or cycles. Thus, the attractor--basin decomposition turns policy behavior into a structural object: success corresponds to flow into the goal basin, while failure corresponds to flow into non-goal basins. Formal definitions of goal attractors, spurious fixed points, non-goal cycles, and basins are collected in Appendix~\ref{app:definitions}.

\paragraph{Finite-horizon evaluation.}
The functional graph describes the eventual closed-loop structure induced by \(f_g\), while evaluation is horizon-limited. We therefore distinguish the eventual goal basin \(\mathcal{B}_g(\{g\})\) from the finite-horizon success set
\[
\mathcal{B}_{g,H}
=
\{s\in\mathcal{S}: \exists t\leq H \text{ such that } f_g^t(s)=g\}.
\]

Reported success is finite-horizon, whereas the functional graph describes eventual closed-loop structure. Specifically, evaluation counts start states that reach the goal within the rollout horizon \(H\), while the attractor--basin decomposition describes where trajectories eventually flow under repeated application of \(f_g\). Thus a state may belong to the eventual goal basin but still be counted as unsuccessful if it reaches the goal only after more than \(H\) steps. In our experiments, reported success always uses the same horizon as training; Appendix~\ref{app:finite_horizon} gives the formal distinction between finite-horizon success sets and eventual basins.

\section{Structural Readout}
\label{sec:structural_readout}

We apply the framework by extracting two linked readouts from each trained greedy policy. The first is the standard outcome readout: for each goal \(g\), we roll out the policy from every valid start state and compute finite-horizon success \(\mathrm{Succ}_H(g)\). The second is the structural readout: for the same policy and goal, we enumerate the one-step greedy map \(f_g(s)\), construct the induced functional graph, and compute local goal support, basin structure, competing attractors, and fragmentation. Thus the rollout readout measures what succeeds, while the graph readout explains how the policy organizes trajectories to produce success or failure.

\paragraph{Structural Metrics}

Table~\ref{tab:structural_metrics} summarizes the main quantities used in the analysis. A fuller list of recorded graph quantities, including attractor counts and types, is provided in Appendix~\ref{app:full_structural_metrics}; implementation details are in Appendix~\ref{app:implementation_details}.

\begin{table}[t]
	\centering
	\footnotesize
	\setlength{\tabcolsep}{3.5pt}
	\renewcommand{\arraystretch}{1.08}
	\caption{Structural quantities computed from the goal-indexed functional graph.}
	\label{tab:structural_metrics}
	\begin{tabular}{lll}
		\toprule
		Metric & Definition & Interpretation \\
		\midrule
		Success &
		\(\mathrm{Succ}_H(g)=|\mathcal{B}_{g,H}\setminus\{g\}|/(|\mathcal{S}|-1)\) &
		Starts reaching \(g\) within \(H\) \\[3pt]
		
		LGS &
		\(\mathrm{LGS}(g)=|N(g)|^{-1}\sum_{s\in N(g)}\mathbf{1}\{f_g(s)=g\}\) &
		Local entry support \\[3pt]
		
		Goal basin &
		\(\beta_g=|\mathcal{B}_g(\{g\})|/|\mathcal{S}|\) &
		Goal-attractor dominance \\[3pt]
		
		Competing basin &
		\(\mathrm{Comp}(g)=\max_{A\in\mathcal{A}_g,A\neq\{g\}}|\mathcal{B}_g(A)|/|\mathcal{S}|\) &
		Largest non-goal basin \\[3pt]
		
		Fragmentation &
		\(\mathrm{Frag}(g)=1-\sum_{A\in\mathcal{A}_g}p_A^2\) &
		Dispersion across basins \\
		\bottomrule
	\end{tabular}
	
	\vspace{0.3em}
	\begin{minipage}{0.96\linewidth}
		\scriptsize
		\emph{Note.} \(f_g\) is the greedy successor map; \(N(g)\) is the valid one-step neighborhood of \(g\); \(\mathcal{A}_g\) is the attractor set; \(\mathcal{B}_g(A)\) is the basin of attractor \(A\); \(\mathcal{B}_{g,H}\) is the horizon-\(H\) success set; and \(p_A=|\mathcal{B}_g(A)|/|\mathcal{S}|\). Full definitions and implementation details are in Appendix~\ref{app:full_structural_metrics}.
	\end{minipage}
\end{table}

For boundary, corner, and wall-adjacent goals, \(N(g)\) includes only valid neighboring states. If no non-goal attractor exists, \(\mathrm{Comp}(g)=0\). We report goal-basin fraction as a structural summary, but the explanatory analyses focus on local goal support, largest competing basin, and fragmentation rather than using goal-basin size itself as a predictor of success. One immediate consequence of the local-support definition is that its zero case is not merely descriptive: in four-neighbor GridWorlds, it gives an exact obstruction to goal entry. 

\begin{proposition}[Zero local support precludes goal entry]
	\label{prop:zero_lgs}
	Consider a deterministic four-neighbor GridWorld and a deterministic greedy successor map \(f_g:\mathcal{S}\to\mathcal{S}\) for goal \(g\). Let \(N(g)\) denote the set of valid one-step neighbors of \(g\), excluding walls and invalid states. If
	\[
	f_g(s)\neq g \qquad \forall s\in N(g),
	\]
	then no trajectory starting from any \(s_0\neq g\) can reach \(g\). Consequently, when evaluation starts exclude the goal state, \(\mathrm{LGS}(g)=0\) implies \(\mathrm{Succ}_H(g)=0\) for every horizon \(H\).
\end{proposition}

\begin{proof}
	Any trajectory that reaches \(g\) from \(s_0\neq g\) must have a first hitting time \(t\) such that \(s_t=g\) and \(s_{t-1}\neq g\). Under four-neighbor movement, the predecessor state \(s_{t-1}\) must be a valid one-step neighbor of \(g\), so \(s_{t-1}\in N(g)\). If \(\mathrm{LGS}(g)=0\), then \(f_g(s)\neq g\) for every \(s\in N(g)\), making the final transition into \(g\) impossible. Hence no non-goal start can reach \(g\), and \(\mathrm{Succ}_H(g)=0\) for every horizon \(H\).
\end{proof}

Proposition~\ref{prop:zero_lgs} gives a local, asymmetric diagnostic. Zero local support rules out goal entry from all non-goal starts, but positive or even full local support does not guarantee global success. A policy may route all immediate neighbors into the goal while distant states are still captured by competing attractors, cycles, or fragmented basins. Thus weak local support is a hard failure signal, whereas success requires global basin organization in the full policy-induced graph.

The next section tests this diagnostic empirically and then uses the global graph metrics to characterize the residual failure modes that local support alone cannot resolve.

\subsection{Policy Maps of Learned Functional Graphs}
\label{sec:policy_maps}

Before aggregating over all seed--goal pairs, we inspect representative policy-induced functional graphs. For a fixed goal \(g\), the learned greedy policy defines the one-step map \(f_g\), visualized as a vector field over states. The red star marks the goal, orange markers indicate non-goal attractors, and the statistics below each panel report finite-horizon success \((S)\), local goal support \((\mathrm{Sup})\), largest competing basin \((C)\), and fragmentation \((F)\).

Figure~\ref{fig:policy_maps} illustrates three policy-induced functional graphs from the same trained policy. The left panel shows a goal-dominant case: trajectories organize into a coherent basin flowing into the goal (3,4), with perfect success, full local support, and no competing basin. The middle panel shows a competition case: although the goal (4,3) is spatially close to the successful goal, a large non-goal attractor captures most states, producing low success despite partial local support. The right panel shows fragmented failure: the goal (7,3) has no local support, success is zero, and trajectories are split across multiple non-goal attractors.

These examples show what aggregate success rates hide. Nearby goals trained under the same protocol can induce sharply different functional graphs: one becomes a stable goal attractor, while another is displaced by competing or fragmented non-goal basins. Thus geometric proximity alone does not determine success; the relevant object is the policy-induced flow over the state space. The following sections quantify these patterns over all seeds, goals, and robustness settings. Additional policy-map examples and detailed interpretations are provided in Appendix~\ref{app:additional_maps}.

\begin{figure}[t]
	\centering
	\includegraphics[width=0.92\linewidth]{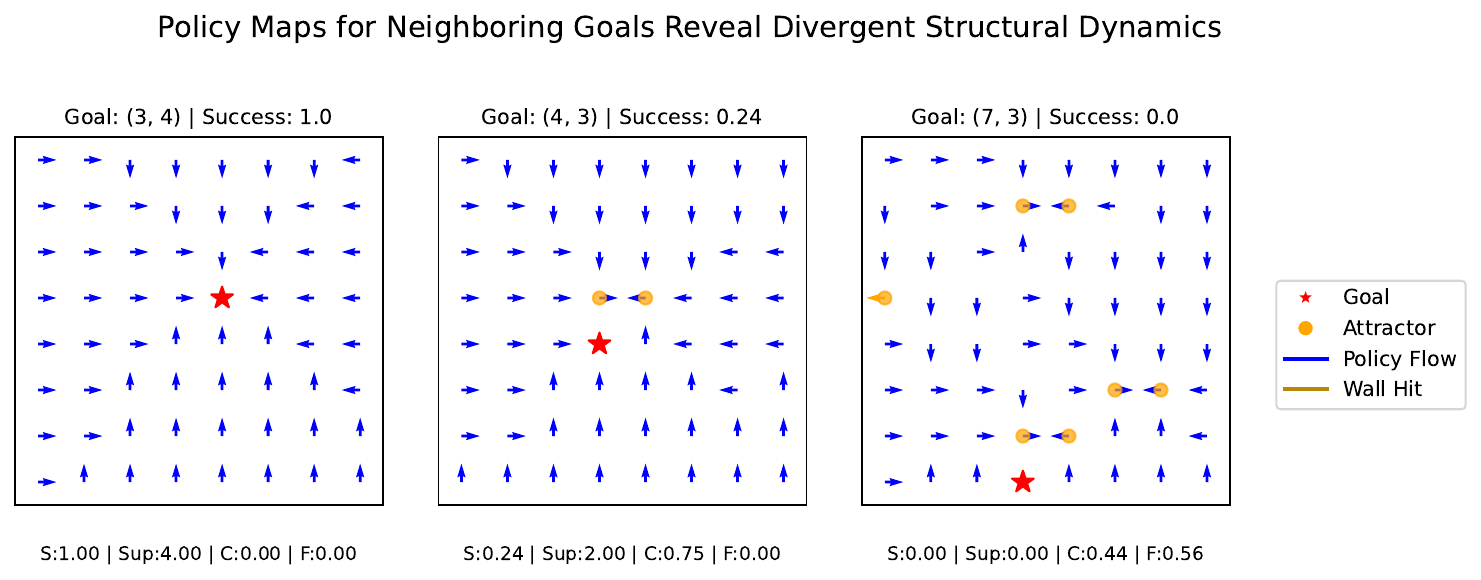}
	\caption{
		Policy-induced functional graphs from one \(8{\times}8\) TD-Uniform policy, seed \(0\), for goals \((3,4)\), \((4,3)\), and \((7,3)\). Arrows show the greedy one-step map \(f_g\); the red star marks the goal; orange markers denote non-goal attractors. Panel text reports success \(S\), local support count \(\mathrm{Sup}\), largest competing basin \(C\), and fragmentation \(F\). The panels illustrate goal-dominant, competitor-dominated, and fragmented regimes.
	}
	\label{fig:policy_maps}
\end{figure}

The maps show that success is governed not by geometry alone, but by the policy-induced structure over the state space: coherent goal basins, competing non-goal attractors, and fragmentation across multiple basins. Additional policy maps in Appendix~\ref{app:additional_maps} show that the same regimes recur across learning rules, sampling schemes, and environment variants. Qualitatively, MC policies tend to exhibit weaker organization and more fragmented flow than TD policies, while the \(12{\times}12\) and bottleneck settings preserve the same structural failure modes at larger scale or under geometric constraints.

\paragraph{Policy Maps in Bottleneck-Constrained Environments}
\label{sec:bn_policy_maps}

We also visualize policy-induced graphs in an \(8{\times}8\) bottleneck environment, where walls and a narrow passage impose stronger geometric constraints than the open grid. This tests whether the same structural readout remains informative when successful trajectories must route through constrained regions.   Figure~\ref{fig:bn_policy_maps} shows representative TD policy maps. The high-success goal \((6,2)\) forms a largely coherent goal-directed basin with full local support, whereas lower-success goals such as \((5,6)\) and \((6,7)\) exhibit competing attractors, fragmented flow, and wall-hit behavior. Thus, even under constrained geometry, success and failure are explained by the same policy-induced structures: local entry into the goal, coherent routing through the state space, competition from non-goal attractors, and fragmentation across basins. The panels illustrate that the same structural regimes persist under constrained geometry: coherent goal-directed routing for high-success goals, and competing attractors, fragmented flow, or wall-hit behavior for low-success goals. Additional bottleneck examples are provided in Appendix~\ref{app:BN_maps}.

\begin{figure}[t]
	\centering
	\includegraphics[width=\linewidth]{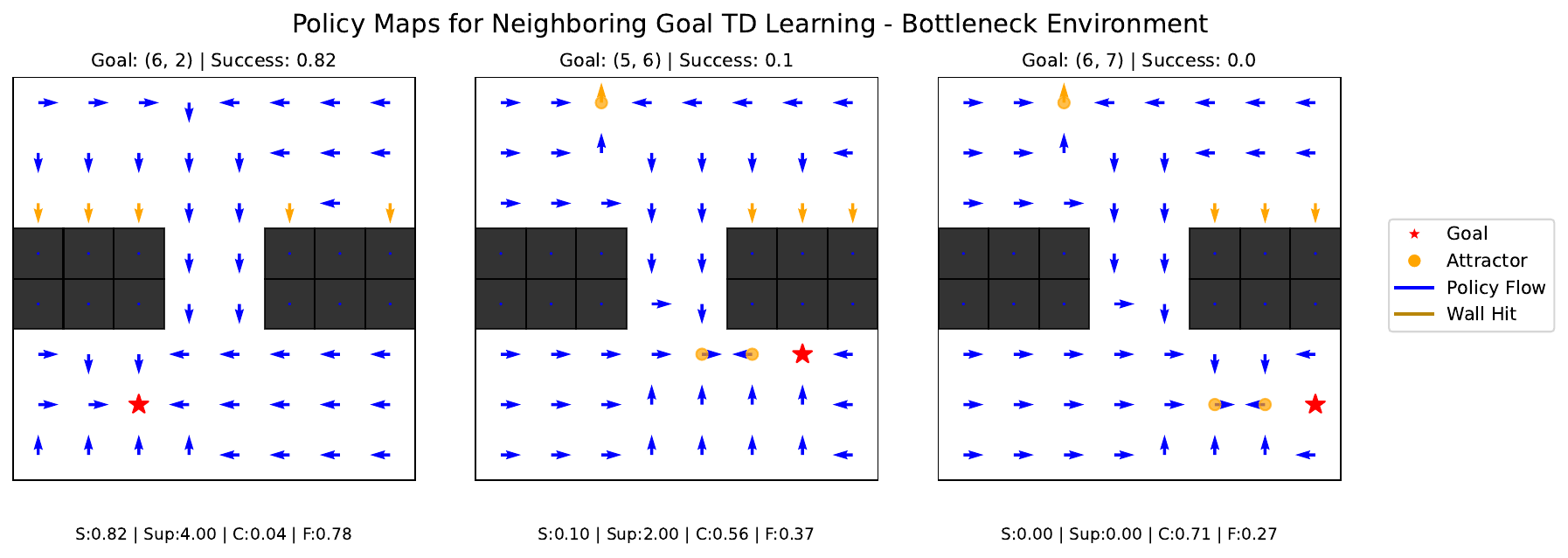}
	\caption{
		Policy maps in the \(8{\times}8\) bottleneck environment under TD learning. The red star marks the goal, orange markers indicate non-goal attractors, blue arrows show greedy policy flow, and brown arrows indicate wall-hit transitions. Text beneath each panel reports \(S\): finite-horizon success, \(\mathrm{Sup}\): local goal support count, \(C\): largest competing-basin fraction, and \(F\): fragmentation index. 
	}
	\label{fig:bn_policy_maps}
\end{figure}
\section{From Local Support to Structural Failure Modes}
\label{sec:local_to_failure}

Proposition~\ref{prop:zero_lgs} shows that zero local support exactly prevents goal entry in deterministic four-neighbor GridWorlds. We now ask how far this local diagnostic extends beyond the zero-support boundary. We first test whether weak local support identifies low-success goals across trained policies, then use global graph metrics to characterize the residual failures that retain nonzero local support.

\paragraph{Local support as a rule-based failure diagnostic.}
We use LGS directly as a diagnostic, not as a learned classifier. For each seed--goal pair, define a low-success failure as \(\mathrm{Succ}_H(g)<0.25\). For threshold \(\tau\), the rule-based prediction is
\[
\widehat{y}_{\mathrm{fail}}(g;\tau)
=
\mathbf{1}\{\mathrm{LGS}(g)\leq \tau\}.
\]
No parameters are fitted: the same threshold is applied directly to the local support value. The full threshold sweep is reported in Appendix Table~\ref{tab:app_lgs_threshold_sweep}.

\paragraph{Local support stratifies global structure.}
The fixed threshold results show that weak LGS predicts low-success goals. For the main 8$\times$8 TD setting, the rule \(\mathrm{LGS}\leq 0.5\) achieves precision 0.921, recall 0.929, and F1 0.925; across all five settings, F1 is at least 0.881 (Appendix Table~\ref{tab:app_lgs_threshold_sweep}). We next ask why this local quantity is informative by relating LGS bands to global graph structure. Across all five settings, low-support goals have low success, small goal basins, and large competing basins; full-support goals have much larger goal basins and smaller competitors.

Table~\ref{tab:lgs_stratification_ranges} summarizes this pattern by grouping exact LGS values into three bands---low support, high partial support, and full support---and reporting the range of mean outcomes across settings. The full condition-specific stratification is provided in Appendix~\ref{app:lgs_stratification}.

\begin{table}[t]
	\centering
	\footnotesize
	\setlength{\tabcolsep}{3.5pt}
	\renewcommand{\arraystretch}{0.90}
	\caption{LGS stratification across five settings. Ranges are computed across conditions after grouping exact LGS values into three support bands.}
	\label{tab:lgs_stratification_ranges}
	\begin{tabular}{lcccc}
		\toprule
		LGS band & Succ. & Fail \% & Goal basin & Comp. basin \\
		\midrule
		Low $(\leq 0.5)$        & .056--.125 & 83.4--93.9 & .063--.139 & .546--.778 \\
		Partial $(0.5,1)$       & .247--.498 & 27.1--70.4 & .252--.506 & .351--.491 \\
		Full $(=1)$             & .586--.952 & 0.7--11.3  & .593--.952 & .041--.280 \\
		\bottomrule
	\end{tabular}
\end{table}

The pattern is consistent across settings. Low-support goals have very low success, high failure rates, small goal basins, and large competing basins. Full-support goals show the reverse pattern, while high partial support lies between these regimes. Thus, increasing LGS is associated with a local-to-global transition from competitor-dominated failure toward goal-basin dominance. However, LGS is asymmetric: it identifies likely failures from local information near the target, but it does not fully characterize global success. Nonzero, and even full, local support can coexist with large competing basins or fragmented dynamics. We therefore use the full policy-induced functional graph as a post-hoc diagnostic to classify the structural form of each seed--goal outcome. This taxonomy is not a fitted predictive model; it summarizes recurring graph structures in the learned policies and separates failures due to weak local access, dominant competitors, and fragmentation.
\paragraph{Post-hoc taxonomy of structural failure modes.}
\label{sec:taxonomy}

LGS is intentionally local: it tests whether the learned policy provides immediate entry support into the goal. It does not fully determine global success, since a goal may have nonzero or even full local support while distant states are captured by non-goal attractors or split across multiple basins. We therefore use the full policy-induced functional graph as a post-hoc diagnostic of each seed--goal outcome.

The taxonomy is rule-based rather than learned. It uses the goal-basin fraction \(\beta_g\), largest competing basin \(\mathrm{Comp}(g)\), and fragmentation index \(\mathrm{Frag}(g)\) to assign each seed--goal graph to one of four regimes: goal-dominant, competitor-dominated, fragmented, or partial/contested. Specifically, a graph is classified as goal-dominant if \(\beta_g\geq 0.75\); competitor-dominated if \(\mathrm{Comp}(g)\geq 0.5\) and \(\mathrm{Comp}(g)>\beta_g\); fragmented if \(\mathrm{Frag}(g)\geq 0.3\), \(\beta_g<0.75\), and \(\mathrm{Comp}(g)<0.5\); and partial/contested otherwise.

\begin{table}[t]
	\centering
	\footnotesize
	\setlength{\tabcolsep}{3.2pt}
	\renewcommand{\arraystretch}{0.90}
	\caption{Post-hoc taxonomy of policy-induced graph structure across all seed--goal graphs.}
	\label{tab:taxonomy_summary}
	\begin{tabular}{lrrrrrrr}
		\toprule
		Regime & \(n\) & \% & Succ. & LGS & Goal & Comp. & Frag. \\
		\midrule
		Goal-dominant        & 1351 & 17.4 & .944 & .978 & .944 & .037 & .104 \\
		Partial/contested    &  308 &  4.0 & .554 & .767 & .561 & .392 & .097 \\
		Competitor-dominated & 4196 & 54.1 & .073 & .334 & .084 & .733 & .195 \\
		Fragmented           & 1905 & 24.5 & .244 & .537 & .253 & .326 & .565 \\
		\bottomrule
	\end{tabular}
\end{table}

Table~\ref{tab:taxonomy_summary} shows that the regimes separate sharply. Goal-dominant graphs have high success and large goal basins, while competitor-dominated graphs have near-zero success and large non-goal basins. Fragmented graphs have higher success than competitor-dominated graphs, but their state mass is dispersed across multiple basins rather than organized around the goal. Thus, LGS provides a local screen for likely failure, while the taxonomy explains the global form of residual failure.
\section{Discussion and Limitations}
\label{sec:discussion}

Our experiments deliberately use small deterministic sparse-reward GridWorlds. The goal is not environmental realism, but exact diagnosis: after training, a deterministic greedy goal-conditioned policy induces a successor map over states, organizing behavior into goal basins, competing attractors, and fragmented failure regions. This controlled setting lets us enumerate the full policy graph and connect aggregate success to explicit structural mechanisms.

The framework has two uses. Local goal support provides a cheap screen for likely low-performing goals without a full rollout census. The full functional-graph readout then diagnoses why failures occur, separating weak local entry from competitor dominance and fragmentation.

The main limitation is that the analysis assumes deterministic greedy evaluation, finite discrete states, and local movement geometry. Extending it to stochastic policies, continuous states, richer actions, and high-dimensional observations will require approximate graph construction or local dynamical diagnostics. A natural next step is to scale geometry in a controlled sequence---from open grids to walls, bottlenecks, multi-room layouts, and MiniGrid-style tasks---to test whether the same structural failure modes persist. 
 
 More broadly, the decomposition suggests that different failures may require different interventions: improving local entry, weakening competing attractors, or reshaping exploration and curricula to reduce fragmentation.
\bibliographystyle{plainnat}
\bibliography{refs}

%%%%%%%%%%%%%%%%%%%%%%%%%%%%%%%%%%%%%%%%%%%%%%%%%%%%%%%%%%%%

\appendix

\section{Technical appendices and supplementary material}

\appendix

\section{Detailed Experimental Setup}
\label{app:experimental_setup}

This appendix provides additional details on the environments, training configuration, curricula, and evaluation protocol used in the experiments. The main paper reports the structural diagnostics computed after training; here we specify how the trained policies and evaluation data were generated.

\subsection{Environment Details}
\label{app:environment_details}

We use finite deterministic GridWorld environments with sparse goal-reaching rewards. A state is a grid coordinate \(s=(x,y)\), and a goal is another valid coordinate \(g=(g_x,g_y)\). The valid state set \(\mathcal{S}\) contains all grid cells except obstacle cells in the bottleneck environment. Unless otherwise stated, starts and goals are sampled from valid states and evaluation excludes the trivial case \(s=g\).

The action space is four-neighbor movement:
\[
\mathcal{A}=\{\text{up},\text{down},\text{left},\text{right}\}.
\]
Each action attempts to move the agent by one grid cell in the corresponding cardinal direction. The transition function is deterministic. If the proposed next state is outside the grid boundary, or would enter an obstacle cell in the bottleneck environment, the agent remains in its current state. Thus, boundary and wall interactions can induce self-loops under particular greedy policies.

Rewards are sparse. The agent receives reward \(1\) when it reaches the goal and reward \(0\) otherwise:
\[
r(s,a,g)=\mathbf{1}\{T(s,a)=g\}.
\]
Episodes terminate when the goal is reached or when the horizon \(H\) is exhausted. We use \(H=16\) for the open \(8{\times}8\) environment and \(H=24\) for the open \(12{\times}12\) and bottleneck \(8{\times}8\) environments. The bottleneck environment uses an \(8{\times}8\) grid with obstacle cells forming a constrained passage; all structural metrics and evaluations are computed only over valid non-wall states.

\subsection{Training Configuration}
\label{app:training_config}

Agents are trained with a goal-conditioned value function \(V_\theta(s,g)\), following the UVFA formulation. The input to the network is the concatenation of the current state and goal coordinates:
\[
(s_x,s_y,g_x,g_y).
\]
The value network is a multilayer perceptron with two hidden layers of width \(128\), ReLU activations, and a scalar output:
\[
4 \rightarrow 128 \rightarrow 128 \rightarrow 1.
\]
The same architecture is used across all reported settings.

Training uses either one-step temporal-difference updates or Monte Carlo return targets. For TD learning, the update target is
\[
y_t = r_t + \gamma V_\theta(s_{t+1},g),
\]
with the terminal target equal to the observed terminal reward when the goal is reached. For MC learning, the target is the realized discounted return from the trajectory. In both cases, transitions are collected using \(\epsilon\)-greedy exploration with respect to the learned goal-conditioned value function. During action selection, the agent evaluates the successor state associated with each candidate action and selects the action with highest estimated value with probability \(1-\epsilon\), otherwise selecting a random action.

Unless otherwise specified, all experiments use \(500\) training episodes per seed, discount factor \(\gamma=0.99\), hidden dimension \(128\), replay-buffer training with minibatches, and one gradient-update phase per episode. The exploration rate decays from an initial value of \(1.0\) to a minimum value of \(0.05\). All main conditions are run over \(20\) independent random seeds. Table~\ref{tab:settings} in the main text summarizes the environment, horizon, learning rule, sampling regime, and seed count for each reported condition.

\subsection{Evaluation Protocol}
\label{app:evaluation_protocol}

After training, exploration is disabled and the learned greedy policy is evaluated deterministically. For a fixed goal \(g\), the greedy policy selects the action whose successor state has maximum learned value:
\[
\pi_g(s) \in \arg\max_{a\in\mathcal{A}} V_\theta(T(s,a),g),
\]
with ties resolved by a fixed deterministic rule. This induces the goal-indexed successor map
\[
f_g(s)=T(s,\pi_g(s)).
\]

Evaluation is a full census over valid start--goal pairs. For every valid goal \(g\in\mathcal{S}\), we roll out the greedy policy from every valid start state \(s\neq g\) for at most \(H\) steps. A rollout is counted as successful if it reaches \(g\) within the horizon:
\[
\mathrm{Succ}_H(g)
=
\frac{1}{|\mathcal{S}|-1}
\sum_{s\in\mathcal{S}\setminus\{g\}}
\mathbf{1}\{\exists t\leq H: f_g^t(s)=g\}.
\]
This produces a goal-level success rate for each trained seed and goal. Aggregate evaluation success is then computed by averaging over goals and seeds.

The same deterministic greedy map \(f_g\) is also used to construct the functional graph analyzed in the main paper. For each fixed \(g\), we enumerate \(f_g(s)\) for every valid state \(s\in\mathcal{S}\), identify attractors and basins, and compute the structural quantities in Table~\ref{tab:structural_metrics}. Thus, the rollout readout and the graph readout are computed from the same trained policy: rollouts score finite-horizon success, while the functional graph exposes the closed-loop structure that produces success or failure.

When distance-binned summaries are used, start--goal pairs are grouped by Manhattan distance
\[
d(s,g)=|s_x-g_x|+|s_y-g_y|.
\]
Distance binning is used only for descriptive analysis and does not affect training, evaluation, or the structural taxonomy.

\subsection{Compute Resources}
\label{app:compute_resources}

All experiments were run on standard Google Colab-style CPU runtime environments and did not require GPU acceleration. The environments are finite GridWorlds, the UVFA is a small three-layer MLP, and evaluation is based on deterministic rollouts plus enumeration of finite policy-induced graphs. Individual \(8{\times}8\) training runs complete in minutes on CPU; \(12{\times}12\) and bottleneck runs require longer wall-clock time because they use larger or more constrained state spaces and longer horizons. The reported experiments use 20 seeds per condition, with structural evaluation performed by enumerating all valid start--goal pairs and constructing one greedy successor map per goal. Preliminary and debugging runs required additional compute, but the final reported experiments are lightweight and reproducible on commodity CPU resources without specialized hardware.

\subsection{Configuration Details}
\label{app:configs}

All experiments used the same base training, evaluation, and analysis pipeline, with only the grid layout, horizon, update rule, and curriculum varied across conditions. The shared configuration is summarized in Table~\ref{tab:shared_config}, and the condition-specific settings are summarized in Table~\ref{tab:condition_config}. Unless otherwise stated, all environments used deterministic four-neighbor dynamics, sparse rewards, greedy evaluation, and exhaustive census evaluation over all valid start--goal pairs.

\begin{table}[t]
	\centering
	\caption{Shared experimental configuration across reported runs.}
	\label{tab:shared_config}
	\begin{tabular}{ll}
		\toprule
		Component & Configuration \\
		\midrule
		Reward & Sparse, $+1$ on reaching the goal and $0$ otherwise \\
		Discount factor & $\gamma = 0.99$ \\
		Training episodes & 500 \\
		Seeds & 20 independent random seeds \\
		Exploration & $\epsilon$-greedy \\
		Initial $\epsilon$ & 1.0 \\
		Minimum $\epsilon$ & 0.05 \\
		$\epsilon$ decay & 0.99 per episode \\
		Batch size & 128 \\
		Updates per episode & 1 \\
		UVFA hidden dimension & 128 \\
		UVFA input & Concatenated state--goal pair $(s,g)$ \\
		UVFA output & Scalar goal-conditioned value estimate \\
		Optimizer & Adam \\
		Learning rate & \(10^{-3}\) \\
		Loss & Mean-squared regression loss \\
		Evaluation mode & Exhaustive census evaluation \\
		Evaluation policy & Greedy policy induced by the learned UVFA \\
		Policy-graph analysis & Enabled for all reported runs \\
		Saved outputs & Per-goal metrics and summary CSV files \\
		\bottomrule
	\end{tabular}
\end{table}

\begin{table}[t]
	\centering
	\caption{Condition-specific configurations for the reported experiments.}
	\label{tab:condition_config}
	\begin{tabular}{lllll}
		\toprule
		Experiment ID & Grid / Layout & Horizon & Update rule & Curriculum \\
		\midrule
		\texttt{GW8\_H16\_sparse\_HD128\_E500\_TD} 
		& $8 \times 8$ open grid & 16 & TD(1) & Uniform or edge-biased \\
		
		\texttt{GW8\_H16\_sparse\_HD128\_E500\_MC} 
		& $8 \times 8$ open grid & 16 & Sparse-start MC & Uniform \\
		
		\texttt{GW12\_H24\_sparse\_HD128\_E500\_TD} 
		& $12 \times 12$ open grid & 24 & TD(1) & Uniform \\
		
		\texttt{GW8\_H24\_sparse\_HD128\_E500\_TD\_BN} 
		& $8 \times 8$ bottleneck grid & 24 & TD(1) & Uniform \\
		\bottomrule
	\end{tabular}
\end{table}
 Optimization used Adam with learning rate \(\eta=10^{-3}\) and mean-squared regression loss on TD or Monte Carlo targets with minibatches of size 128; all reported runs used one gradient-update phase per episode.
\section{Curriculum Definitions}
\label{app:curriculum_definitions}

Curricula determine how goals are sampled during training. They do not change the environment, reward function, network architecture, or evaluation protocol. In all cases, evaluation is performed by a full deterministic greedy census over all valid start--goal pairs. Thus, differences between curricula reflect differences in training exposure, not differences in evaluation distribution.

\subsection{Uniform Goal Sampling}
\label{app:uniform_curriculum}

The uniform curriculum samples each valid goal with equal probability. Let \(\mathcal{S}\) denote the set of valid states. Then
\[
p_{\mathrm{unif}}(g)=\frac{1}{|\mathcal{S}|},
\qquad g\in\mathcal{S}.
\]
For each training episode, a start state and goal state are sampled from valid states, excluding the trivial case \(s=g\). This is the baseline training distribution used in the TD-Uniform and MC-Uniform conditions.

\subsection{Edge-Biased Goal Sampling}
\label{app:edge_bias_curriculum}

The edge-biased curriculum increases the probability of sampling goals on the boundary of the grid. Let \(\mathcal{E}\subset\mathcal{S}\) be the set of valid edge states,
\[
\mathcal{E}
=
\{(x,y)\in\mathcal{S}: x=0 \text{ or } x=n-1 \text{ or } y=0 \text{ or } y=n-1\},
\]
and let \(\mathcal{I}=\mathcal{S}\setminus\mathcal{E}\) be the set of valid interior states. The curriculum samples an edge goal with probability \(\rho_{\mathrm{edge}}=0.7\) and an interior goal with probability \(1-\rho_{\mathrm{edge}}=0.3\):
\[
p_{\mathrm{edge}}(g)
=
\begin{cases}
	0.7/|\mathcal{E}|, & g\in \mathcal{E},\\[3pt]
	0.3/|\mathcal{I}|, & g\in \mathcal{I}.
\end{cases}
\]

For the open \(8{\times}8\) grid, \(|\mathcal{E}|=28\) and \(|\mathcal{I}|=36\), so each edge goal is sampled with probability \(0.7/28=0.025\), while each interior goal is sampled with probability \(0.3/36 \approx 0.0083\). Thus edge goals are sampled about three times as often as interior goals. This curriculum is used only during training; evaluation remains a full deterministic census over all valid start--goal pairs.

\paragraph{Comparison.}
Uniform sampling provides an unweighted baseline over goals. Edge-biased sampling tests whether increasing exposure to geometrically harder boundary goals improves learned goal-reaching structure. As reported in the main text, edge bias changes where coherent policy structure forms, but does not eliminate seed variability, competing attractors, or fragmented failures.

\subsection{Seed-Level Variability Under TD-Uniform}
\label{app:seed_variability}

Table~\ref{tab:td_seeds} reports seed-level results for the \(8{\times}8\) TD-uniform setting used in the motivating analysis. Although all seeds share the same environment, architecture, update rule, horizon, and goal-sampling distribution, evaluation success varies substantially. This confirms that the aggregate variability in Table~\ref{tab:main_performance} is not driven by a small number of anomalous runs, but is a persistent seed-level phenomenon.

\begin{table}[h]
	\centering
	\small
	\begin{tabular}{cccccc}
		\toprule
		Seed & Rank & Train & Last100 & Eval & KL \\
		\midrule
		16 & 1  & 0.39 & 0.53 & 0.67 & 0.032 \\
		12 & 2  & 0.39 & 0.51 & 0.65 & 0.032 \\
		2  & 3  & 0.36 & 0.48 & 0.64 & 0.025 \\
		18 & 4  & 0.35 & 0.31 & 0.59 & 0.027 \\
		6  & 5  & 0.38 & 0.49 & 0.53 & 0.029 \\
		7  & 6  & 0.46 & 0.68 & 0.53 & 0.028 \\
		8  & 7  & 0.36 & 0.45 & 0.49 & 0.031 \\
		10 & 8  & 0.42 & 0.44 & 0.48 & 0.036 \\
		0  & 9  & 0.41 & 0.48 & 0.47 & 0.034 \\
		14 & 10 & 0.37 & 0.40 & 0.46 & 0.026 \\
		1  & 11 & 0.40 & 0.50 & 0.43 & 0.045 \\
		17 & 12 & 0.31 & 0.31 & 0.40 & 0.029 \\
		5  & 13 & 0.37 & 0.48 & 0.39 & 0.024 \\
		19 & 14 & 0.35 & 0.36 & 0.39 & 0.026 \\
		15 & 15 & 0.27 & 0.28 & 0.30 & 0.025 \\
		3  & 16 & 0.30 & 0.37 & 0.30 & 0.030 \\
		9  & 17 & 0.31 & 0.36 & 0.30 & 0.029 \\
		13 & 18 & 0.35 & 0.46 & 0.27 & 0.028 \\
		4  & 19 & 0.28 & 0.32 & 0.26 & 0.025 \\
		11 & 20 & 0.41 & 0.49 & 0.26 & 0.030 \\
		\bottomrule
	\end{tabular}
	\caption{Seed-level performance for TD with uniform sampling. Seeds are ranked by evaluation success. Substantial variability is observed across seeds despite similar training statistics.}
	\label{tab:td_seeds}
\end{table}

Despite similar training performance across seeds, evaluation success exhibits substantial variability, ranging from approximately 0.26 to 0.67. 

\subsection{Distance-Based Performance Diagnostics}
\label{app:distance_diagnostics}

As a simple geometric baseline, we also examine success as a function of shortest-path distance between start and goal. Start--goal pairs are grouped into short, medium, and long-distance buckets. Figure~\ref{fig:distance} reports mean success with variability across seeds.

As expected, success generally decreases with distance, and variability is larger for longer-distance pairs. However, distance does not explain the main structural phenomena studied in the paper. Goals at similar geometric difficulty can induce very different policy graphs, including goal-dominant basins, competing non-goal attractors, and fragmented basin decompositions. Thus, distance is a useful baseline descriptor of task difficulty, but it does not replace the policy-induced structural diagnostics developed in the main text.

\begin{figure}[h]
	\centering
	\includegraphics[width=0.7\linewidth]{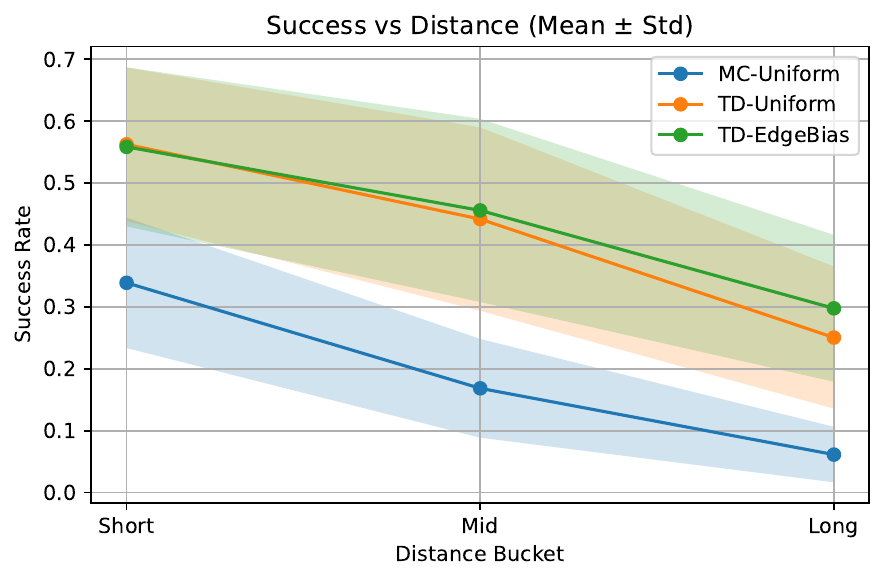}
	\caption{Success rates as a function of goal distance (mean \(\pm\) standard deviation across seeds). Performance decreases with distance for all methods, and variability increases for longer distances. Edge-biased training yields modest improvements at longer distances but does not reduce variability.}
	\label{fig:distance}
\end{figure}

Performance decreases monotonically with distance for all methods, reflecting the increased difficulty of long-horizon navigation. In addition, variability across seeds increases with distance, particularly for long-range goals. While edge-biased training yields modest improvements at longer distances, it does not reduce variability or fundamentally alter overall performance trends. These results reinforce that both difficulty and instability are structured properties of the environment and learned policy.

\subsection{Edge and Interior Goal Performance}
\label{app:edge_interior_performance}

The edge-biased curriculum increases the training probability of boundary goals from their uniform mass of \(28/64=0.4375\) to \(0.7\). To examine whether this additional exposure improves edge-goal performance, we separately evaluate success on edge and interior goals in the open \(8{\times}8\) environment.

\begin{table}[h]
	\centering
	\small
	\caption{Edge versus interior goal success rates in the open \(8{\times}8\) environment. Edge-biased training slightly improves edge-goal success relative to TD-Uniform, but reduces interior-goal success, yielding little net improvement in overall evaluation.}
	\label{tab:edge_interior_success}
	\begin{tabular}{lcc}
		\toprule
		Method & Edge success & Interior success \\
		\midrule
		MC-Uniform    & 0.20 & 0.19 \\
		TD-Uniform    & 0.31 & 0.40 \\
		TD-EdgeBias   & 0.34 & 0.37 \\
		\bottomrule
	\end{tabular}
\end{table}

Table~\ref{tab:edge_interior_success} shows that edge-biased sampling has the intended local effect: edge-goal success increases from \(0.31\) under TD-Uniform to \(0.34\) under TD-EdgeBias. However, interior-goal success decreases from \(0.40\) to \(0.37\). Thus, the curriculum shifts competence toward the emphasized edge region rather than improving the global organization of the learned policy. This supports the main analysis: simple redistribution of goal exposure does not remove competing attractors, fragmented basins, or seed-level variability.

\section{Structural Properties of Finite Policy Maps}
\label{app:finite_maps}

This appendix records the elementary finite-map property used in the main text. For a fixed goal \(g\), greedy evaluation induces the deterministic state map
\[
f_g:\mathcal{S}\to\mathcal{S},
\qquad
f_g(s)=T(s,\pi_g(s)),
\]
where \(\mathcal{S}\) is finite. Equivalently, \(f_g\) defines a directed graph
\[
G_g=(\mathcal{S},E_g),
\qquad
E_g=\{(s,f_g(s)):s\in\mathcal{S}\}.
\]
Every state therefore has exactly one outgoing edge. Such graphs are functional graphs, or mapping digraphs.

\paragraph{Proposition A.1.}
Let \(\mathcal{S}\) be a finite set and let \(f:\mathcal{S}\to\mathcal{S}\) be a deterministic map. The directed graph with vertices \(\mathcal{S}\) and edges \(s\to f(s)\) decomposes into weakly connected components, each containing exactly one directed cycle. Every state not on the cycle has a unique forward path that eventually enters that cycle.

\paragraph{Proof.}
Fix \(s_0\in\mathcal{S}\) and consider its forward orbit
\[
s_0,\ f(s_0),\ f^2(s_0),\ldots .
\]
Since \(\mathcal{S}\) is finite, some state must repeat: there exist \(0\leq i<j\) such that
\[
f^i(s_0)=f^j(s_0).
\]
The repeated segment
\[
f^i(s_0),\ f^{i+1}(s_0),\ldots,\ f^{j-1}(s_0)
\]
therefore forms a directed cycle. Thus every forward orbit eventually enters a cycle.

It remains to show that each weakly connected component contains exactly one cycle. Each state has a unique forward orbit, and hence a unique eventual cycle. Suppose a weakly connected component contained two distinct cycles. Then there would be an undirected path connecting a state whose forward orbit enters one cycle to a state whose forward orbit enters the other. Along this path, there must be two adjacent states \(u\) and \(v\) with different eventual cycles. But the edge between adjacent states is either \(u\to v\) or \(v\to u\). If \(u\to v\), then the forward orbit of \(u\) immediately follows that of \(v\), so \(u\) and \(v\) must have the same eventual cycle, a contradiction. The case \(v\to u\) is identical. Hence a weakly connected component cannot contain two distinct cycles.

Therefore each component contains exactly one directed cycle, and all remaining states in that component have unique forward paths feeding into it. \(\square\)

\paragraph{Interpretation for learned greedy policies.}
Applied to the policy-induced map \(f_g\), the unique directed cycle in each component is an attractor under greedy execution. A self-loop at the goal is the desired goal attractor; a self-loop away from the goal is a spurious fixed point; and a cycle of length greater than one is a non-goal recurrent loop. The states whose forward paths enter an attractor form its basin of attraction. Thus, for each goal \(g\), the learned greedy policy partitions the finite state space into attractor--basin components.
A standard property of finite functional graphs is that each weakly connected component contains exactly one directed cycle, with all remaining states feeding into that cycle. Appendix~\ref{app:finite_maps} recalls this property and explains its interpretation for learned greedy policies.

\subsection{Asymmetry of local support and global success}
\label{app:lgs_asymmetry}

\paragraph{Proposition A.2.}
\textit{Local support is necessary for entry but not sufficient for global success.}
Consider a deterministic GridWorld with local four-neighbor moves and a deterministic goal-indexed successor map \(f_g:\mathcal{S}\to\mathcal{S}\). Let \(N(g)\) be the set of valid one-step neighbors of goal \(g\). For horizons \(H\geq 1\), the following hold.

\begin{enumerate}
	\item If \(\mathrm{LGS}(g)=0\), then \(\mathrm{Succ}_H(g)=0\).
	\item If \(\mathrm{LGS}(g)>0\), then \(\mathrm{Succ}_H(g)>0\) for \(H\geq 1\), but this does not imply high success.
	\item Even \(\mathrm{LGS}(g)=1\) does not guarantee \(\mathrm{Succ}_H(g)=1\).
\end{enumerate}

\paragraph{Proof.}
The first statement follows from Proposition~\ref{prop:zero_lgs}: any trajectory reaching \(g\) from a non-goal start must enter \(g\) from some valid neighbor \(s\in N(g)\). If no neighbor maps to \(g\), such an entry transition is impossible.

For the second statement, if \(\mathrm{LGS}(g)>0\), then there exists at least one neighbor \(s\in N(g)\) such that \(f_g(s)=g\). Since evaluation includes non-goal starts and \(s\neq g\), starting from \(s\) reaches \(g\) in one step. Hence \(\mathrm{Succ}_H(g)>0\) for any \(H\geq 1\). However, this only guarantees that at least one start reaches the goal; it does not constrain the trajectories of other states.

For the third statement, construct a policy-induced graph in which every valid neighbor of \(g\) maps directly to \(g\), so \(\mathrm{LGS}(g)=1\), but all states outside this local neighborhood map to a non-goal fixed point or cycle. Then the goal has full local support, but most starts are captured by a competing attractor and never reach \(g\). Therefore full local support does not imply full finite-horizon success.

\section{Additional Definitions for Policy-Induced Functional Graphs}
\label{app:definitions}

For a fixed goal \(g\), let \(f_g:\mathcal{S}\to\mathcal{S}\) denote the deterministic state map induced by greedy evaluation. An attractor of \(f_g\) is a recurrent set \(A\subseteq\mathcal{S}\) such that trajectories that enter \(A\) remain in \(A\). In the finite deterministic setting considered here, attractors correspond to directed cycles in the functional graph of \(f_g\), including cycles of length one.

\paragraph{Goal attractor.}
For graph construction, we treat the goal as absorbing:
\[
f_g(g)=g.
\]
The goal attractor is the one-state recurrent set
\[
A_g^{\mathrm{goal}}=\{g\}.
\]
A state belongs to the goal basin if its greedy trajectory eventually reaches \(g\).

\paragraph{Spurious fixed point.}
A spurious fixed point is a non-goal state \(s^\star\neq g\) that maps to itself:
\[
f_g(s^\star)=s^\star.
\]
It defines the one-state non-goal attractor
\[
A^{\mathrm{fp}}=\{s^\star\}.
\]
States in its basin are failures under goal \(g\), because their greedy trajectories become trapped at \(s^\star\) rather than reaching the goal.

\paragraph{Non-goal cycle.}
A non-goal cycle of length \(k>1\) is a set of distinct states
\[
A^{\mathrm{cyc}}=\{s_1,\ldots,s_k\}
\]
such that
\[
f_g(s_i)=s_{i+1}\quad \text{for } i=1,\ldots,k-1,
\qquad
f_g(s_k)=s_1,
\]
and \(g\notin A^{\mathrm{cyc}}\). States in its basin are failures under goal \(g\), because their greedy trajectories loop without reaching the goal.

\paragraph{Basin of attraction.}
For any attractor \(A\), its basin under goal \(g\) is
\[
\mathcal{B}_g(A)
=
\{s\in\mathcal{S}: \exists t\geq 0 \text{ such that } f_g^t(s)\in A\}.
\]
Equivalently, \(\mathcal{B}_g(A)\) is the set of states whose greedy trajectories eventually enter \(A\). Because \(f_g\) is deterministic and \(\mathcal{S}\) is finite, each state belongs to exactly one attractor basin.

\subsection{Full Structural Metric Set}

\label{app:full_structural_metrics}

Table~\ref{tab:full_structural_metrics} summarizes the full set of structural quantities computed from the attractor--basin decomposition. The main paper focuses on local goal support, largest competing basin, and fragmentation; the remaining quantities are used for descriptive analysis and robustness checks.

\begin{table}[h]
	\centering
	\small
	\caption{Full set of structural metrics computed from the goal-indexed functional graph. Here \(\mathcal{A}_g\) denotes the set of attractors under \(f_g\), \(\mathcal{B}_g(A)\) is the basin of attractor \(A\), and \(p_A=|\mathcal{B}_g(A)|/|\mathcal{S}|\).}
	\label{tab:full_structural_metrics}
	\begin{tabular}{lll}
		\toprule
		Metric & Definition & Interpretation \\
		\midrule
		
		Goal basin fraction &
		\(\displaystyle \beta_g=\frac{|\mathcal{B}_g(g)|}{|\mathcal{S}|}\) &
		States eventually reaching the goal \\[6pt]
		
		Finite-horizon success &
		\(\displaystyle \mathrm{Succ}_H(g)=\frac{|\mathcal{B}_{g,H}\setminus\{g\}|}{|\mathcal{S}|-1}\) &
		Starts reaching \(g\) within \(H\) \\[6pt]
		
		Failure basin fraction &
		\(\displaystyle \beta_{\mathrm{fail}}=1-\beta_g\) &
		States reaching non-goal attractors \\[6pt]
		
		Largest competing basin &
		\(\displaystyle \beta_{\max}=\max_{A\in\mathcal{A}_g,\,A\neq g}\frac{|\mathcal{B}_g(A)|}{|\mathcal{S}|}\) &
		Dominant non-goal failure basin \\[8pt]
		
		Goal dominance margin &
		\(\displaystyle \Delta_g=\beta_g-\beta_{\max}\) &
		Goal basin advantage over strongest competitor \\[6pt]
		
		Failure concentration &
		\(\displaystyle C_{\mathrm{fail}}=\frac{\beta_{\max}}{\beta_{\mathrm{fail}}}\) &
		Whether failure is concentrated in one basin \\[8pt]
		
		Cycle basin fraction &
		\(\displaystyle \beta_{\mathrm{cyc}}=\sum_{A\in\mathcal{A}_{\mathrm{cyc}}}\frac{|\mathcal{B}_g(A)|}{|\mathcal{S}|}\) &
		Failure mass captured by cycles \\[8pt]
		
		Spurious fixed-point fraction &
		\(\displaystyle \beta_{\mathrm{fp}}=\sum_{A\in\mathcal{A}_{\mathrm{fp}}}\frac{|\mathcal{B}_g(A)|}{|\mathcal{S}|}\) &
		Failure mass captured by non-goal fixed points \\[8pt]
		
		Number of attractors &
		\(\displaystyle |\mathcal{A}_g|\) &
		Structural complexity of policy graph \\[6pt]
		
		Fragmentation &
		\(\displaystyle F=1-\sum_{A\in\mathcal{A}_g}p_A^2\) &
		Dispersion across attractor basins \\[8pt]
		
		Local goal support &
		\(\displaystyle \mathrm{LGS}(g)=\frac{1}{|N(g)|}\sum_{s\in N(g)}\mathbf{1}\{f_g(s)=g\}\) &
		Immediate inward support near goal \\[8pt]
		
		Mean time to attractor &
		\(\displaystyle \frac{1}{|\mathcal{S}|}\sum_{s\in\mathcal{S}}\tau_A(s)\) &
		Average transient length before recurrence \\[8pt]
		
		Mean time to goal &
		\(\displaystyle \frac{1}{|\mathcal{B}_g(g)|}\sum_{s\in\mathcal{B}_g(g)}\tau_g(s)\) &
		Average hitting time among successful states \\
		
		\bottomrule
	\end{tabular}
\end{table}

Note: When the rollout horizon is long enough to reach the goal from all states in the goal basin, \(\beta_g\) coincides with finite-horizon success; otherwise success is the horizon-truncated goal basin fraction.

\subsection{Finite-Horizon Success versus Eventual Basins}
\label{app:finite_horizon}

For a fixed learned greedy policy, the functional graph describes the full closed-loop structure induced by \(f_g\), independent of the rollout horizon used to score evaluation. Evaluation, however, is horizon-limited. We therefore distinguish the eventual goal basin,
\[
\mathcal{B}_g(\{g\})
=
\{s\in\mathcal{S}: \exists t\geq 0 \text{ such that } f_g^t(s)=g\},
\]
from the horizon-truncated success set,
\[
\mathcal{B}_{g,H}
=
\{s\in\mathcal{S}\setminus\{g\}: \exists t\leq H \text{ such that } f_g^t(s)=g\}.
\]
Since \(\mathcal{B}_{g,H}\subseteq \mathcal{B}_g(\{g\})\), finite horizons can classify as failures states that would eventually reach the goal only after more than \(H\) steps. In our experiments, the same horizon used during training is also used for evaluation scoring, so reported success is computed from \(\mathcal{B}_{g,H}\). Attractors and basins, by contrast, describe the full induced policy structure once the learned greedy policy is fixed. The distinction between horizon-truncated success and eventual goal basins is detailed in Appendix~\ref{app:finite_horizon}.

\section{Implementation Details}
\label{app:implementation_details}

\subsection{Attractor and Basin Computation}

For each trained seed and evaluation goal \(g\), we construct the deterministic successor map \(f_g\) over all valid states. In open grids, the valid state set is the full grid. In bottleneck environments, wall cells are excluded from the valid state set. For each valid state \(s\), we evaluate the greedy action selected by the learned goal-conditioned policy and record the resulting successor \(f_g(s)\).

Since \(f_g\) is deterministic and the state space is finite, repeated application of \(f_g\) from any state eventually reaches an attractor. Attractors may be:
\begin{itemize}
	\item the goal fixed point \(\{g\}\),
	\item a non-goal spurious fixed point \(\{s^\star\}\), or
	\item a non-goal cycle of length greater than one.
\end{itemize}

We identify attractors by forward iteration. Starting from each state, we repeatedly apply \(f_g\) until either a previously assigned state is reached or a repeated state appears within the current trajectory. A repeated state identifies a cycle; all states preceding the cycle are assigned to the basin of that cycle. If the cycle is the goal self-loop, these states belong to the goal basin; otherwise, they belong to a non-goal failure basin.

This procedure assigns every valid state to exactly one attractor basin, yielding the partition
\[
\mathcal{S}
=
\bigsqcup_{A\in\mathcal{A}_g}\mathcal{B}_g(A),
\]
where \(\mathcal{A}_g\) is the set of attractors under \(f_g\). We also compute the transient distance from each state to its attractor, defined as the number of applications of \(f_g\) required to first enter the attractor.
\section{Policy Maps and Structural Visualization}
\label{app:policy_maps}

To better understand variability in goal-reaching behavior, we analyze the structure of the learned policy over the state space. For a fixed goal, we represent the learned greedy policy as a \emph{policy map}, which assigns to each state its greedy action or, equivalently, its next state under the induced map \(f_g\). This yields a vector field over the state space.

Policy maps provide a direct visualization of the induced functional graph. They reveal how trajectories organize globally, including coherent goal-directed flow, competing attractors, basin boundaries, and fragmented regions. These structural properties help explain differences in evaluation performance across seeds and goals that are not apparent from aggregate metrics alone.

\subsection{Detailed Interpretation of Policy-Map Examples}

\paragraph{Goal success via basin formation.}
Figure~\ref{fig:policy_maps} shows that for goal $(3,4)$ (left), the policy induces a single dominant basin of attraction centered at the goal. From nearly every state, trajectories align with a coherent flow that leads toward the goal. Crucially, the goal has strong \emph{local support}: all neighboring states transition directly into the goal. This creates a stable terminal attractor. As a result, trajectories from across the grid join a global flow that reliably converges to the goal, yielding perfect success.

\paragraph{Competing attractors and partial failure.}
For the nearby goal $(4,3)$ (middle), the structure changes qualitatively. A cyclic attractor emerges in the upper region of the grid, with a substantial basin of attraction that captures a large fraction of trajectories. This competing structure diverts trajectories away from the goal. Additionally, the goal exhibits weak local support—only a subset of neighboring states lead into it—limiting its ability to capture nearby trajectories. Consequently, only a restricted region of the state space (primarily the lower-left) successfully reaches the goal, leading to partial success.

\paragraph{Fragmentation and complete failure.}
For goal $(7,3)$ (right), the policy exhibits a highly fragmented structure. Multiple competing attractors partition the state space into disjoint regions, none of which are dominated by the goal. The goal itself fails to form a meaningful basin of attraction and lacks local support. As a result, trajectories are trapped in alternative attractors, and the probability of reaching the goal collapses to zero.

\section{Additional policy maps}
\label{app:additional_maps}
\paragraph{Additional examples under TD learning.}
Figure~\ref{fig:td_policy_maps_additional} provides further examples of policy maps under TD learning. While all three goals achieve moderate to high success, their underlying structures differ in important ways.

The goal $(5,3)$ exhibits a relatively coherent basin of attraction with strong local support, allowing trajectories from most regions to converge reliably. In contrast, the goal $(6,4)$ shows partial competition from a nearby attractor, reducing the effective basin size and limiting success. The goal $(7,6)$, despite high success, demonstrates mild fragmentation, with small regions of the state space diverted toward alternative attractors.

These examples highlight that even within a single learning regime, performance is not determined solely by the presence of a goal-directed flow, but by the balance between basin formation, competing structures, and fragmentation. This reinforces the need for quantitative metrics to systematically capture these structural properties.

\begin{figure}[t]
	\centering
	\includegraphics[width=0.95\linewidth]{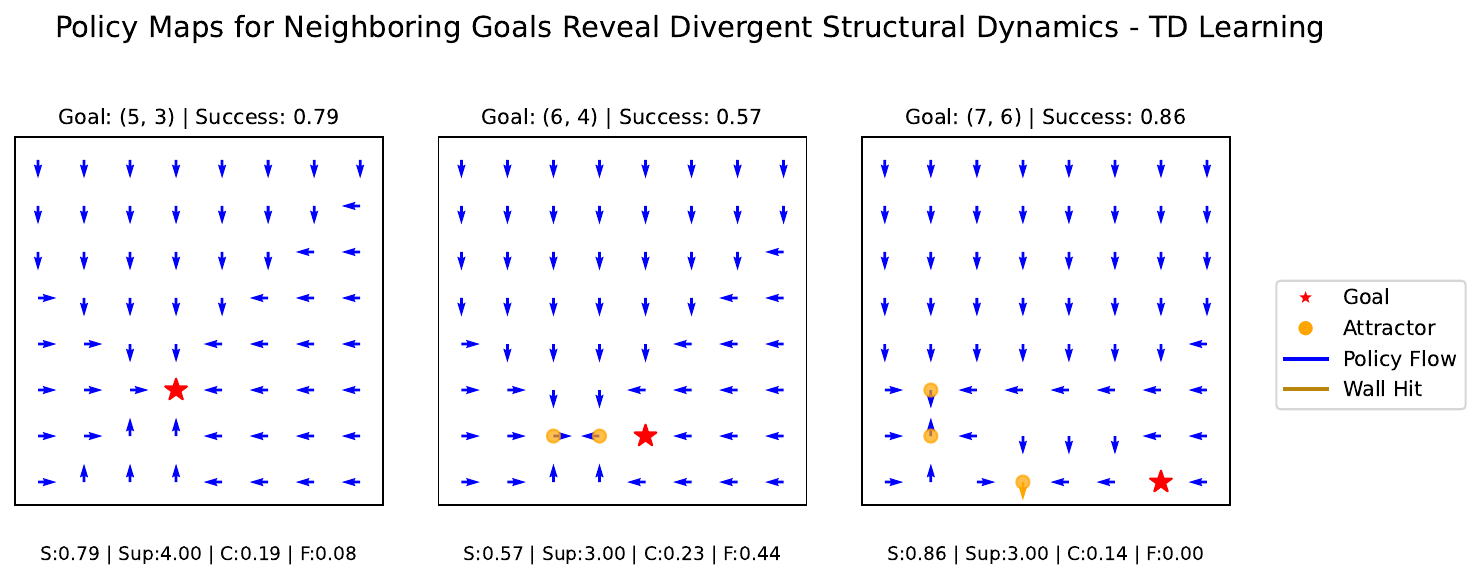}
	\caption{
		Additional policy maps under TD learning illustrating variability within the same learning regime. 
		Even among relatively successful goals, differences in local support, competing attractors, and fragmentation are evident. 
		While some goals exhibit well-formed basins of attraction, others show partial competition or mild fragmentation, leading to reduced success. 
		These examples highlight that performance variability persists within TD learning and is governed by structural differences in the induced policy dynamics.
	}
	\label{fig:td_policy_maps_additional}
\end{figure}
\paragraph{Policy structure under Monte Carlo learning.}
Figure~\ref{fig:mc_policy_maps} illustrates policy maps obtained under Monte Carlo (MC) training for three representative goals. While the same qualitative structural elements—attractors, basins, and competing dynamics—are present, the overall organization of the policy is noticeably weaker compared to TD learning.

Even for the high-performing goal $(6,3)$, the basin of attraction is less coherent and local support is incomplete, with not all neighboring states directing into the goal. As a result, trajectories from parts of the state space fail to converge reliably. 

For the medium- and low-performing goals, competing attractors and cyclic structures dominate larger portions of the grid. These structures capture trajectories away from the goal, reducing the effective basin size. In the lowest-performing case, the state space is highly fragmented, with multiple competing regions and no dominant goal-directed flow.

These patterns suggest that MC learning, due to delayed credit assignment, is less effective at propagating value information across the state space. This results in weaker basin formation, reduced local support, and increased fragmentation, providing a structural explanation for the lower success rates observed under MC training.

\begin{figure}[t]
	\centering
	\includegraphics[width=0.95\linewidth]{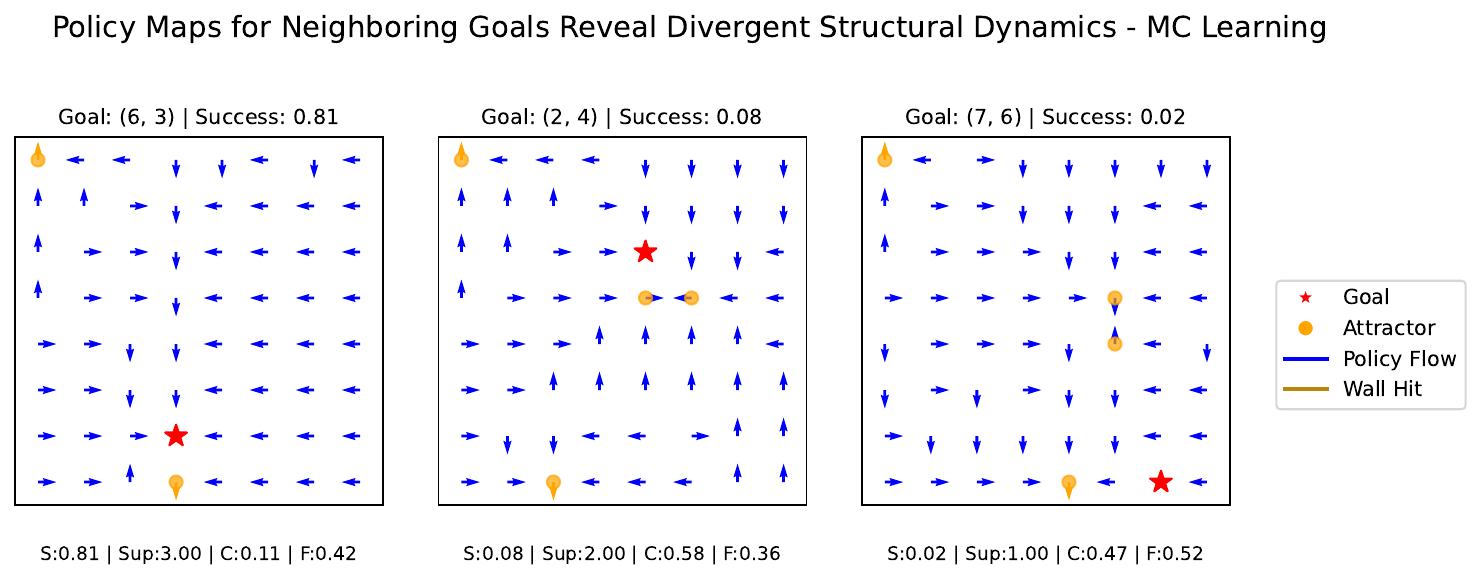}
	\caption{
		Policy maps under Monte Carlo (MC) learning for three representative goals (high, medium, and low success). 
		Compared to TD learning, MC policies exhibit weaker basin formation and reduced local support even for high-performing goals. 
		Competing attractors and cyclic structures are more prevalent, leading to fragmented dynamics and lower overall success rates.
	}
	\label{fig:mc_policy_maps}
\end{figure}

\paragraph{Policy Maps Under Edge Bias Curriculum Sampling}
\label{app:edgebias_maps}

We next examine whether curriculum-based sampling alters the structural dynamics of the learned policy. Figure~\ref{fig:policy_maps_edgebias} shows representative policy maps under an edge-biased curriculum. While the curriculum increases exposure to edge goals, the resulting policies continue to exhibit the same structural patterns observed under uniform sampling.

In particular, successful goals still correspond to the formation of strong basins of attraction with high local support, while unsuccessful goals are characterized by competing attractors and fragmentation. Notably, the curriculum does not uniformly improve performance across edge goals: some goals become highly optimized, while others remain poorly structured. This indicates that curriculum primarily redistributes where structure forms, rather than fundamentally changing the underlying mechanisms governing success. This suggests that curriculum interventions alone are insufficient to overcome structural limitations in sparse goal-conditioned learning.

\begin{figure}[t]
	\centering
	\includegraphics[width=0.95\linewidth]{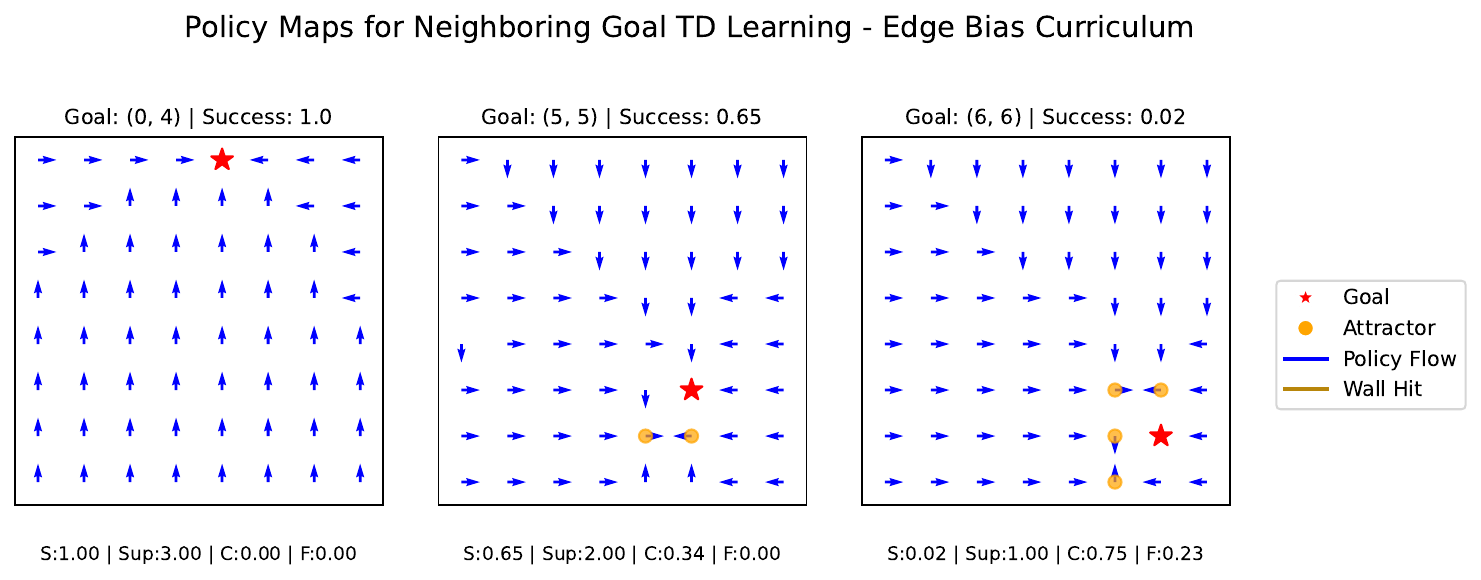}
	\caption{
		Policy maps under TD learning with edge-biased curriculum sampling. 
		Despite preferential sampling of edge goals, structural variability persists. 
		The goal $(0,4)$ (left) forms a strong basin of attraction with full local support, yielding perfect success. 
		The goal $(5,5)$ (middle) exhibits partial competition from a nearby attractor, reducing its effective basin. 
		The goal $(6,6)$ (right) remains fragmented, with competing attractors preventing goal convergence. 
		These examples show that curriculum shifts the distribution of learning but does not fundamentally alter the structural determinants of performance.
	}
	\label{fig:policy_maps_edgebias}
	\end{figure}
		
\section{Policy Maps in Extended Environments}

\subsection{Policy Maps in 12 x 12 Grid Environment}
\label{app:policy_maps_12x12}

\paragraph{Policy Maps - 12x12 grid - TD Learning}
We further analyze policy structure in a $12 \times 12$ GridWorld to assess whether the observed dynamics extend beyond the base setting. Figure~\ref{fig:policy_12x12} presents policy maps for three representative goals using TD learning

The successful goal (left) exhibits a large, coherent basin of attraction, with policy arrows from most states aligning toward the goal. Local support is strong, with all neighboring states directing flow into the goal, resulting in stable convergence.

The intermediate goal (center) achieves partial success. While a portion of the state space directs flow toward the goal, a competing attractor captures a significant region, particularly in the lower portion of the grid. This competing structure diverts trajectories away from the goal, limiting its effective basin of attraction.

The unsuccessful goal (right) displays a highly fragmented structure, with multiple competing attractors and minimal local support. The state space is partitioned into several regions that do not lead to the goal, resulting in negligible success despite the deterministic policy.

These observations mirror the patterns seen in the $8 \times 8$ setting, indicating that basin formation, competition, and fragmentation remain the key determinants of performance even as the state space grows. The persistence of these structures suggests that the relationship between policy dynamics and goal success is not an artifact of small environments, but reflects a more general property of goal-conditioned learning.

\begin{figure}[t]
	\centering
	\includegraphics[width=0.95\linewidth]{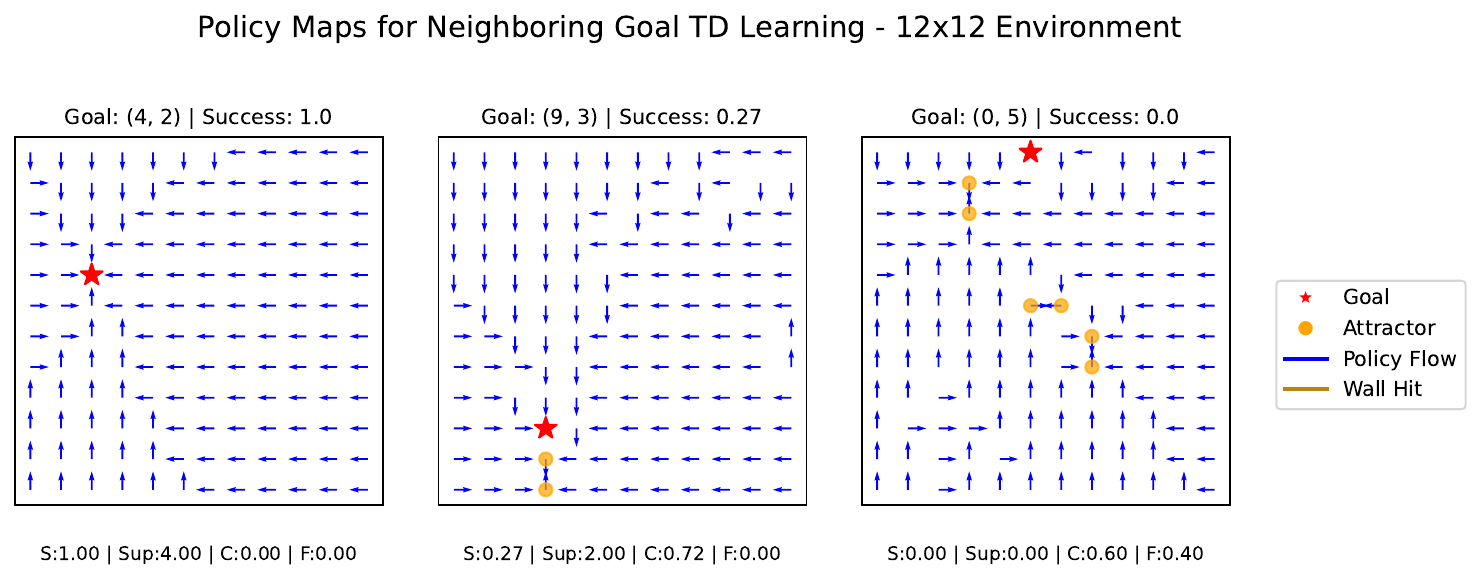}
	\caption{
		Policy maps for neighboring goals in a $12 \times 12$ environment under TD learning. 
		The left goal exhibits a coherent basin of attraction with full local support and achieves perfect success. 
		The middle goal shows partial success due to a competing attractor that captures a large portion of the state space. 
		The right goal exhibits fragmented dynamics with multiple competing attractors and weak local support, leading to failure.
	}
	\label{fig:policy_12x12}
\end{figure}

\paragraph{Monte Carlo dynamics in 12x12 grid}
In the $12 \times 12$ setting, Monte Carlo learning exhibits increased fragmentation relative to TD. Instead of forming a single dominant competing basin, the state space is partitioned into multiple smaller attractors, resulting in diffuse and less predictable trajectories. This leads to weaker basin formation even for moderately successful goals and highlights the sensitivity of MC learning to environment scale.

\begin{figure}[t]
	\centering
	\includegraphics[width=0.95\linewidth]{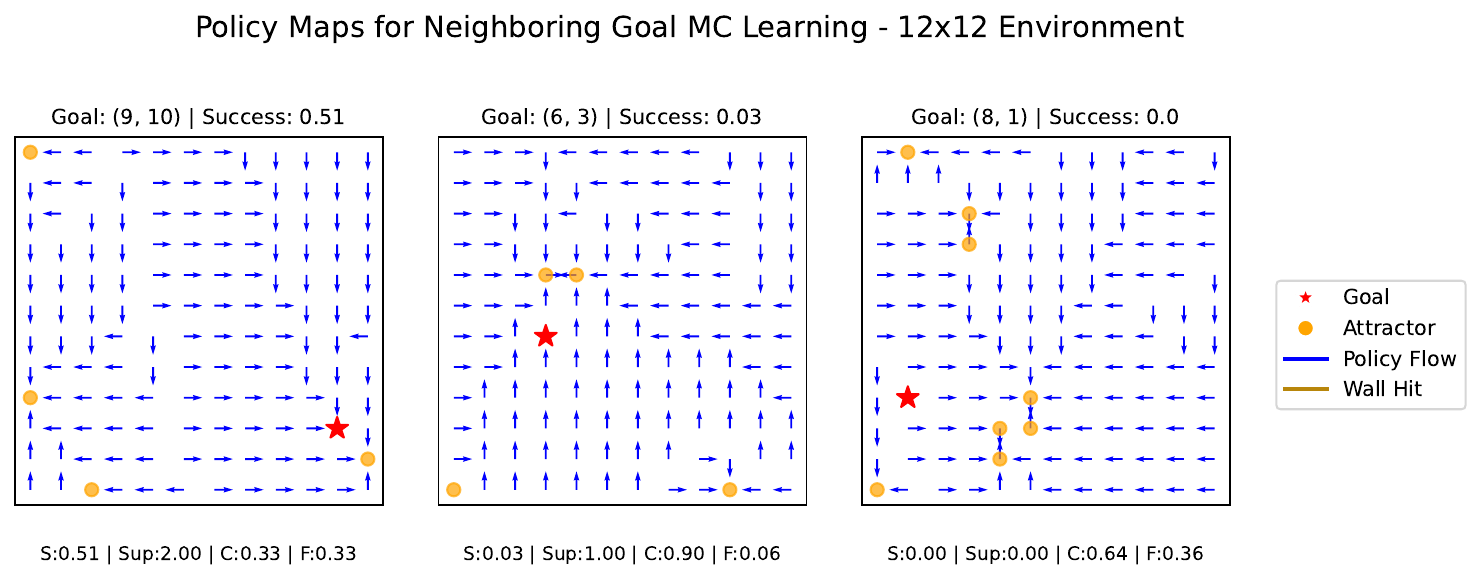}
	\caption{
		Policy maps for neighboring goals under Monte Carlo (MC) learning in a $12 \times 12$ environment. 
		Compared to TD learning, MC policies exhibit more fragmented dynamics, with multiple smaller attractors and weaker basin formation. 
		Successful goals show partial and less coherent basins, while failures arise from diffuse fragmentation rather than a single dominant competing basin.
	}
	\label{fig:policy_12x12_mc}
\end{figure}

\subsection{Policy Maps in BN Grid Environment}
\label{app:BN_maps}

Figure~\ref{fig:bn_policy_maps_appendix} provides additional policy-map examples from the \(8{\times}8\) bottleneck environment. The bottleneck setting introduces walls and a narrow passage, so successful goal-reaching requires the learned policy to route trajectories around blocked regions rather than relying on the more symmetric connectivity of the open grid.

The left panel shows a high-success goal. Although the fragmentation index is nonzero, the learned policy routes most states toward the lower part of the grid and into the goal region, yielding high finite-horizon success. The middle panel shows a low-success goal located near the bottleneck: despite nonzero local support, trajectories are diverted into nearby non-goal attractors and wall-hit behavior, producing a small effective goal basin. The right panel shows a zero-success goal with no local support; trajectories are captured by non-goal attractors and wall-induced flow patterns rather than entering the goal. These examples illustrate that the same structural ingredients observed in open grids---local support, competing attractors, wall-induced fixed points, and fragmentation---remain informative when the geometry is constrained.

\begin{figure}[t]
	\centering
	\includegraphics[width=0.95\linewidth]{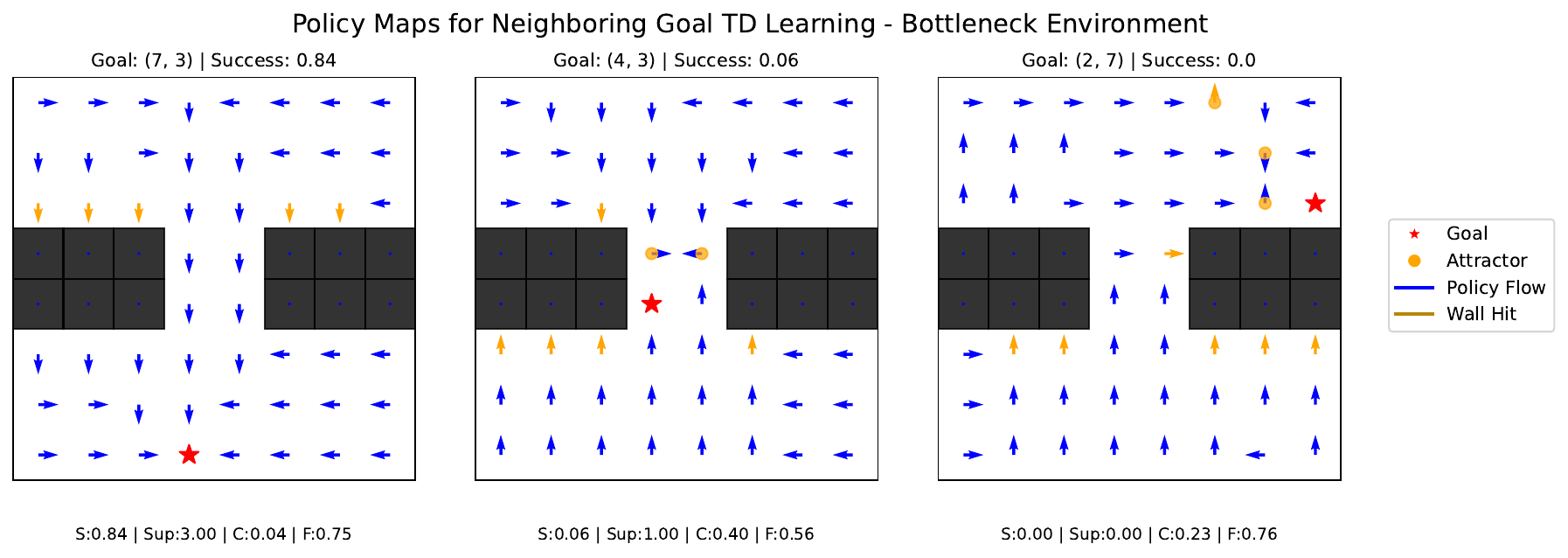}
	\caption{
		Additional policy maps in the \(8{\times}8\) bottleneck environment under TD learning. 
		The red star marks the goal, orange markers indicate non-goal attractors, blue arrows show greedy policy flow, and brown arrows indicate wall-hit transitions. 
		Text below each panel reports \(S\): finite-horizon success, \(\mathrm{Sup}\): local goal support count, \(C\): largest competing-basin fraction, and \(F\): fragmentation index. 
		The examples show that constrained geometry introduces wall-hit behavior and bottleneck routing failures, but the policy-induced graph still separates high-success, low-success, and zero-success goals through the same attractor--basin structure.
	}
	\label{fig:bn_policy_maps_appendix}
\end{figure}

\section{Local Support Diagnostics}
\label{app:lgs_diagnostics}

This appendix provides additional diagnostics for local goal support (LGS). We first report ranking diagnostics based on AUC and Spearman correlation, then analyze a deterministic threshold rule for classifying low-success goals, and finally show how LGS levels stratify global policy-graph structure.

\subsection{Continuous LGS Ranking Diagnostics}
\label{app:lgs_continuous_diagnostics}

We first ask whether LGS provides a useful ranking of goals within each trained policy. Table~\ref{tab:app_lgs_per_seed_robustness} reports two complementary diagnostics computed separately for each seed and then summarized across seeds. AUC measures how well \(-\mathrm{LGS}\) ranks low-success goals ahead of non-failing goals; a value of \(1\) indicates perfect ranking and \(0.5\) corresponds to random ranking. Spearman's \(\rho\) measures the rank correlation between LGS and finite-horizon goal success within a seed. High positive values indicate that goals with larger local support tend to have higher success.

The mean values summarize the typical within-seed diagnostic strength, while the minimum values report the weakest seed-level case. This guards against the possibility that the pooled seed--goal result is driven only by a subset of well-behaved trained policies.

\begin{table}[t]
	\centering
	\small
	\caption{Per-seed robustness of the local goal support diagnostic. AUC and Spearman correlation are computed separately within each trained seed, then summarized across seeds for each condition.}
	\label{tab:app_lgs_per_seed_robustness}
	\begin{tabular}{lrrrr}
		\toprule
		Condition & Mean AUC & Min AUC & Mean $\rho$ & Min $\rho$ \\
		\midrule
		$8{\times}8$ TD uniform     & 0.965 & 0.921 & 0.911 & 0.845 \\
		$8{\times}8$ MC uniform     & 0.908 & 0.801 & 0.871 & 0.634 \\
		$8{\times}8$ TD edge        & 0.946 & 0.828 & 0.923 & 0.742 \\
		$8{\times}8$ BN TD uniform  & 0.963 & 0.889 & 0.890 & 0.676 \\
		$12{\times}12$ TD uniform   & 0.950 & 0.899 & 0.870 & 0.714 \\
		\bottomrule
	\end{tabular}
\end{table}

Across all settings, LGS remains a strong within-policy diagnostic. Even the weakest seed-level AUC remains well above random ranking, and the mean AUC is above \(0.90\) in every condition. Thus, the relationship between LGS and goal success is not only a pooled population effect across seed--goal rows; it is also visible within individual trained policies.

\subsection{Threshold Sweep for Rule-Based Failure Classification}
\label{app:lgs_threshold_sweep}

We next evaluate LGS as a deterministic failure rule. This is not a learned classifier. For a seed--goal pair, let \(y_{\mathrm{fail}}=1\) denote a low-success goal, defined as
\[
\mathrm{Succ}_H(g)<0.25.
\]
For a threshold \(\tau\), define the rule-based prediction
\[
\widehat{y}_{\mathrm{fail}}(g;\tau)
=
\mathbf{1}\{\mathrm{LGS}(g)\leq \tau\}.
\]
Thus, a goal is predicted to fail if its local goal support is at or below the threshold. No parameters are fitted; the sweep simply evaluates the same deterministic diagnostic at different operating points.

The threshold \(\tau\) is varied over the discrete LGS values induced by local neighborhood geometry. At \(\tau=0\), the rule predicts failure only for goals with zero local support. This is highly conservative: it has perfect precision in all tested settings, but low recall because many failures have weak nonzero support. Increasing \(\tau\) labels more goals as predicted failures, which typically increases recall while reducing precision. Table~\ref{tab:app_lgs_threshold_sweep} reports this precision--recall tradeoff. The main text uses the fixed threshold \(\tau=0.5\), which provides a robust balance across all settings.

\begin{table}[t]
	\centering
	\small
	\caption{Threshold sweep for rule-based LGS failure classification. For each threshold \(\tau\), a seed--goal pair is predicted to be a low-success failure when \(LGS \leq \tau\). Failure is defined as \(\mathrm{Succ}_H(g)<0.25\).}
	\label{tab:app_lgs_threshold_sweep}
	\resizebox{\textwidth}{!}{%
		\begin{tabular}{lrrrrr}
			\toprule
			Condition & \(\tau\) & Precision & Recall & F1 & Accuracy \\
			\midrule
			$8{\times}8$ TD uniform & 0.00 & 1.000 & 0.171 & 0.292 & 0.569 \\
			$8{\times}8$ TD uniform & 0.25 & 0.986 & 0.320 & 0.483 & 0.644 \\
			$8{\times}8$ TD uniform & 0.33 & 0.988 & 0.611 & 0.755 & 0.794 \\
			$8{\times}8$ TD uniform & 0.50 & 0.921 & 0.929 & 0.925 & 0.922 \\
			$8{\times}8$ TD uniform & 0.67 & 0.839 & 0.973 & 0.901 & 0.889 \\
			$8{\times}8$ TD uniform & 0.75 & 0.800 & 0.994 & 0.887 & 0.868 \\
			\midrule
			$8{\times}8$ MC uniform & 0.00 & 1.000 & 0.268 & 0.422 & 0.457 \\
			$8{\times}8$ MC uniform & 0.25 & 0.987 & 0.466 & 0.633 & 0.599 \\
			$8{\times}8$ MC uniform & 0.33 & 0.983 & 0.654 & 0.786 & 0.735 \\
			$8{\times}8$ MC uniform & 0.50 & 0.918 & 0.865 & 0.891 & 0.843 \\
			$8{\times}8$ MC uniform & 0.67 & 0.885 & 0.950 & 0.917 & 0.872 \\
			$8{\times}8$ MC uniform & 0.75 & 0.861 & 0.976 & 0.915 & 0.865 \\
			\midrule
			$8{\times}8$ TD edge & 0.00 & 1.000 & 0.165 & 0.283 & 0.588 \\
			$8{\times}8$ TD edge & 0.25 & 0.981 & 0.322 & 0.484 & 0.662 \\
			$8{\times}8$ TD edge & 0.33 & 0.970 & 0.618 & 0.755 & 0.802 \\
			$8{\times}8$ TD edge & 0.50 & 0.834 & 0.933 & 0.881 & 0.876 \\
			$8{\times}8$ TD edge & 0.67 & 0.759 & 0.984 & 0.857 & 0.838 \\
			$8{\times}8$ TD edge & 0.75 & 0.745 & 0.990 & 0.850 & 0.828 \\
			\midrule
			$8{\times}8$ BN TD uniform & 0.00 & 1.000 & 0.262 & 0.415 & 0.607 \\
			$8{\times}8$ BN TD uniform & 0.25 & 1.000 & 0.372 & 0.542 & 0.665 \\
			$8{\times}8$ BN TD uniform & 0.33 & 0.977 & 0.686 & 0.806 & 0.824 \\
			$8{\times}8$ BN TD uniform & 0.50 & 0.924 & 0.872 & 0.897 & 0.893 \\
			$8{\times}8$ BN TD uniform & 0.67 & 0.830 & 0.980 & 0.899 & 0.883 \\
			$8{\times}8$ BN TD uniform & 0.75 & 0.807 & 0.984 & 0.887 & 0.866 \\
			\midrule
			$12{\times}12$ TD uniform & 0.00 & 1.000 & 0.148 & 0.258 & 0.332 \\
			$12{\times}12$ TD uniform & 0.25 & 0.998 & 0.549 & 0.709 & 0.646 \\
			$12{\times}12$ TD uniform & 0.33 & 0.998 & 0.730 & 0.843 & 0.787 \\
			$12{\times}12$ TD uniform & 0.50 & 0.939 & 0.903 & 0.921 & 0.878 \\
			$12{\times}12$ TD uniform & 0.67 & 0.933 & 0.950 & 0.942 & 0.908 \\
			$12{\times}12$ TD uniform & 0.75 & 0.910 & 0.999 & 0.952 & 0.922 \\
			\bottomrule
		\end{tabular}
	}
\end{table}

The sweep shows the expected precision--recall tradeoff. The zero-support rule has perfect precision in every setting, matching Proposition~\ref{prop:zero_lgs}, but low recall because many failures have weak nonzero support. Intermediate thresholds recover more failures. The fixed threshold \(\tau=0.5\) provides a stable balance across all five settings, with F1 scores above \(0.88\). Higher thresholds further increase recall but begin to reduce precision, especially in the edge-biased condition.

\subsection{Exact and compressed LGS stratification}
\label{app:lgs_stratification}

Table~\ref{tab:app_lgs_stratification_compressed} shows how local support stratifies the global policy graph. We compress exact LGS values into three bands: low support, high partial support, and full support. For each band, we report mean finite-horizon success, the percentage of low-success goals, mean goal-basin fraction, largest competing basin, and fragmentation. This table is descriptive rather than a separate predictive model: it shows how a local property of the policy near the goal is associated with the global organization of trajectories.

\begin{table}[t]
	\centering
	\small
	\caption{Compressed stratification by local goal support. Exact LGS levels are aggregated into low support, high partial support, and full support. Means are weighted by the number of seed--goal graphs in each exact LGS level. Low-support goals have small goal basins and large competing basins, while full-support goals usually have large goal basins and low failure rates.}
	\label{tab:app_lgs_stratification_compressed}
	\resizebox{\textwidth}{!}{%
		\begin{tabular}{llrrrrrr}
			\toprule
			Condition & LGS band & $n$ & Mean succ. & Failure \% & Goal basin & Comp. basin & Frag. \\
			\midrule
			$8{\times}8$ TD uniform & Low support $(LGS \leq 0.5)$ & 672 & 0.095 & 92.1 & 0.109 & 0.760 & 0.143 \\
			$8{\times}8$ TD uniform & High partial $(0.5 < LGS < 1)$ & 155 & 0.483 & 27.7 & 0.491 & 0.415 & 0.155 \\
			$8{\times}8$ TD uniform & Full support $(LGS=1)$ & 453 & 0.915 & 0.9 & 0.916 & 0.068 & 0.051 \\
			\midrule
			$8{\times}8$ MC uniform & Low support $(LGS \leq 0.5)$ & 894 & 0.072 & 91.8 & 0.087 & 0.546 & 0.399 \\
			$8{\times}8$ MC uniform & High partial $(0.5 < LGS < 1)$ & 182 & 0.268 & 57.7 & 0.279 & 0.415 & 0.415 \\
			$8{\times}8$ MC uniform & Full support $(LGS=1)$ & 204 & 0.586 & 11.3 & 0.593 & 0.280 & 0.251 \\
			\midrule
			$8{\times}8$ TD edge & Low support $(LGS \leq 0.5)$ & 706 & 0.125 & 83.4 & 0.139 & 0.778 & 0.094 \\
			$8{\times}8$ TD edge & High partial $(0.5 < LGS < 1)$ & 133 & 0.498 & 27.1 & 0.506 & 0.453 & 0.066 \\
			$8{\times}8$ TD edge & Full support $(LGS=1)$ & 441 & 0.952 & 1.4 & 0.952 & 0.041 & 0.015 \\
			\midrule
			$8{\times}8$ BN TD uniform & Low support $(LGS \leq 0.5)$ & 523 & 0.074 & 92.4 & 0.092 & 0.560 & 0.386 \\
			$8{\times}8$ BN TD uniform & High partial $(0.5 < LGS < 1)$ & 152 & 0.356 & 40.8 & 0.369 & 0.351 & 0.472 \\
			$8{\times}8$ BN TD uniform & Full support $(LGS=1)$ & 365 & 0.632 & 2.5 & 0.639 & 0.168 & 0.542 \\
			\midrule
			$12{\times}12$ TD uniform & Low support $(LGS \leq 0.5)$ & 2172 & 0.056 & 93.9 & 0.063 & 0.654 & 0.300 \\
			$12{\times}12$ TD uniform & High partial $(0.5 < LGS < 1)$ & 307 & 0.247 & 70.4 & 0.252 & 0.491 & 0.357 \\
			$12{\times}12$ TD uniform & Full support $(LGS=1)$ & 401 & 0.829 & 0.7 & 0.830 & 0.081 & 0.305 \\
			\bottomrule
		\end{tabular}
	}
\end{table}

Across all settings, low-support goals have very low success, small goal basins, and large competing basins. Full-support goals have much larger goal basins and substantially lower failure rates. The TD regimes show the sharpest transition: under TD-uniform and TD-edge, full support corresponds to mean success above \(0.91\) and almost no low-success failures. MC and bottleneck policies exhibit the same ordering but remain more globally contested, with larger residual competing basins or fragmentation even when local support is complete. Thus, LGS captures a robust local-to-global ordering, while competition and fragmentation explain why local support is not always sufficient for full success.

\subsection{Threshold sweep for LGS failure classification}
\label{app:lgs_threshold_sweep}

The threshold analysis uses LGS as a rule-based diagnostic, not as a learned classifier. For a seed--goal pair, let \(y_{\mathrm{fail}}=1\) denote a low-success goal, defined as \(\mathrm{Succ}_H(g)<0.25\). For a threshold \(\tau\), we define the predicted failure label as
\[
\widehat{y}_{\mathrm{fail}}(g;\tau)
=
\mathbf{1}\{\mathrm{LGS}(g)\leq \tau\}.
\]
Thus, a goal is predicted to fail if its local goal support is at or below the chosen threshold. No parameters are fitted; the sweep simply evaluates the same deterministic diagnostic at different operating points.

The threshold \(\tau\) is varied over the discrete LGS values induced by the local neighborhood structure. At \(\tau=0\), the rule predicts failure only for goals with zero local support. This is a highly conservative rule: it has perfect or near-perfect precision because zero local support precludes goal entry, but low recall because many failing goals have weak nonzero support. Increasing \(\tau\) labels more goals as predicted failures. This typically increases recall, because more true failures are detected, but reduces precision, because some partially successful goals are also included.

For each threshold, we report precision, recall, F1, and accuracy with respect to the low-success failure label. Precision is the fraction of predicted failures that are actual failures, recall is the fraction of actual failures that are detected, F1 is the harmonic mean of precision and recall, and accuracy is the fraction of correctly classified seed--goal pairs.

Table~\ref{tab:app_lgs_threshold_sweep} reports the full threshold sweep for the deterministic rule above. The main paper reports the fixed threshold \(\tau=0.5\), which gives high F1 across all conditions. The longer raw output additionally records TP, FP, FN, and TN counts and is retained in the analysis outputs for auditability.
\begin{figure*}[t]
	\centering
	\includegraphics[width=\textwidth]{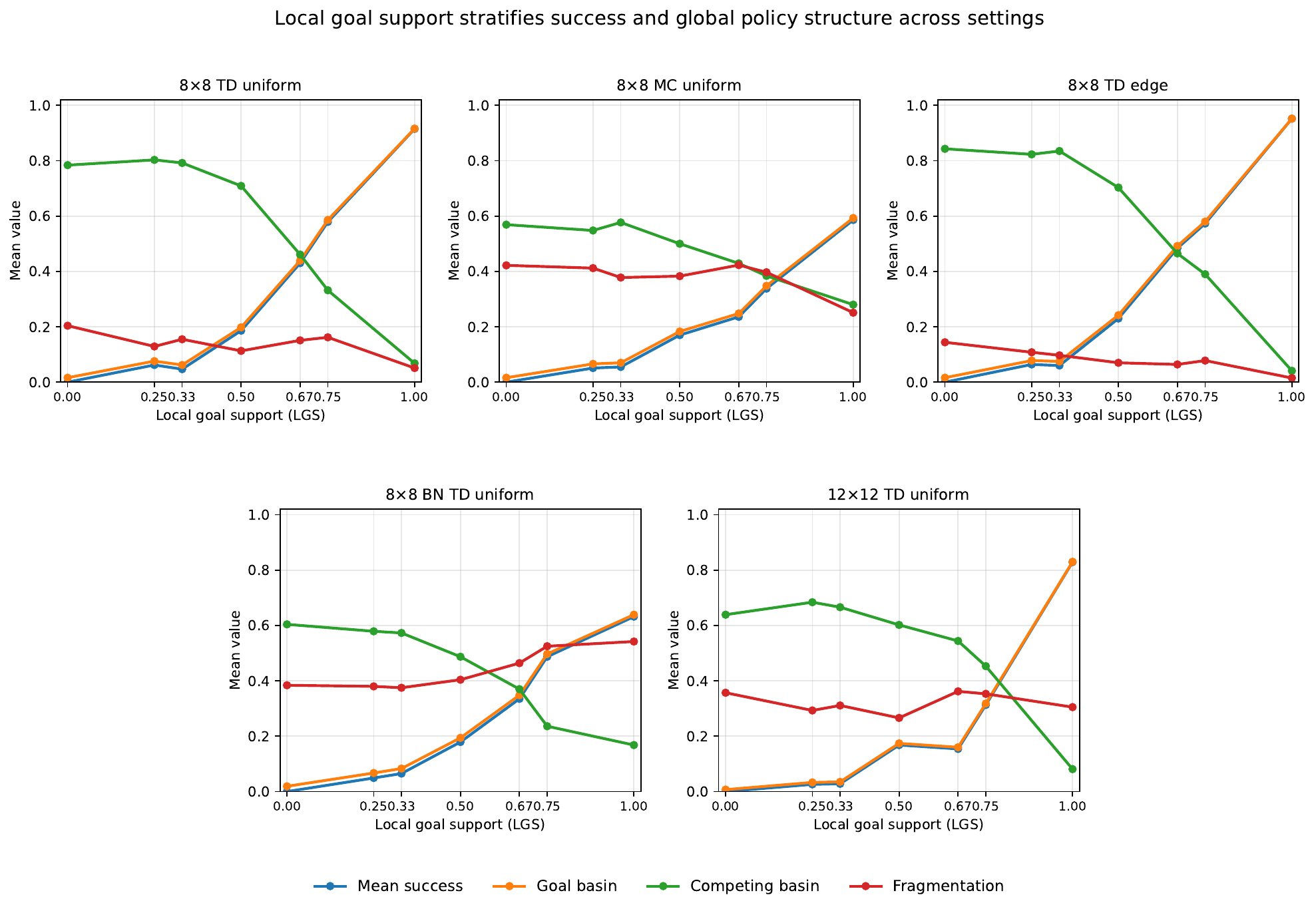}
	\caption{LGS stratification curves across experimental settings. Each panel plots mean finite-horizon success, goal-basin fraction, largest competing basin, and fragmentation against exact local goal support. Increasing LGS is generally associated with a transition from competitor-dominated failure to goal-basin dominance, though MC and bottleneck settings retain more residual competition or fragmentation.}
	\label{fig:app_lgs_stratification_all}
\end{figure*}

Across settings, increasing LGS corresponds to a consistent local-to-global transition. Low-support goals have low success, small goal basins, and large competing basins. Full-support goals usually have much higher success and substantially larger goal basins, with competing basins reduced. The transition is sharpest in the TD settings, where full local support almost always corresponds to near-complete goal dominance. In contrast, MC and bottleneck policies retain more residual competition or fragmentation, showing that local support is highly informative but not sufficient for complete success. These residual cases motivate the additional structural failure taxonomy based on competing attractors and fragmentation.
Overall, LGS stratifies not only rollout success but also the global organization of the induced policy graph.

\paragraph{Full Local Support Does Not Guarantee Global Success.}
Proposition $A.2$ shows that local support has an asymmetric implication: zero support rules out goal entry, but full local support does not guarantee global success. Table~\ref{tab:app_full_lgs4_weak_success} illustrates this empirically by selecting interior goals with full local support but weak or partial global success. All listed cases have local support count \(4\), \(\mathrm{LGS}=1\), and \(\mathrm{Succ}_H(g)<0.75\). Thus, every immediate neighbor of the goal maps directly into the goal, but the learned policy still fails from a substantial fraction of starts.
\begin{table}[t]
	\centering
	\small
	\caption{Full local support with weak or partial global success. All cases are interior goals with local support count \(4\), \(\mathrm{LGS}=1\), and \(\mathrm{Succ}_H(g)<0.75\). Even when every immediate neighbor maps into the goal, failures persist because states outside the local neighborhood can be captured by competing or fragmented basins. Regime columns report percentages within these full-support weak-success cases.}
	\label{tab:app_full_lgs4_weak_success}
	\begin{tabular}{lrrrrrrrr}
		\toprule
		Condition & \(n\) & Succ. & Goal basin & Comp. basin & Frag. & Comp.-dom. \% & Frag. \% & Partial \% \\
		\midrule
		$8{\times}8$ TD uniform    &  34 & 0.587 & 0.594 & 0.324 & 0.210 &  8.8 & 32.4 & 50.0 \\
		$8{\times}8$ TD edge       &  16 & 0.478 & 0.486 & 0.437 & 0.121 & 37.5 & 12.5 & 43.8 \\
		$8{\times}8$ MC uniform    &  73 & 0.519 & 0.526 & 0.297 & 0.359 &  5.5 & 63.0 & 30.1 \\
		$12{\times}12$ TD uniform  & 119 & 0.593 & 0.596 & 0.210 & 0.495 &  3.4 & 68.1 & 23.5 \\
		$8{\times}8$ BN TD uniform &  80 & 0.607 & 0.614 & 0.179 & 0.558 &  1.2 & 87.5 &  8.8 \\
		\bottomrule
	\end{tabular}
\end{table}

\subsection{Rule-based structural taxonomy}
\label{app:structural_taxonomy}

The structural taxonomy in Section~\ref{sec:taxonomy} is intended as a post-hoc interpretive diagnostic, not as a learned classifier. After training, each fixed goal \(g\) induces a deterministic functional graph under greedy execution. For each seed--goal graph, we summarize this graph using three quantities: the goal-basin fraction \(\beta_g\), the largest competing basin \(\mathrm{Comp}(g)\), and the fragmentation index \(\mathrm{Frag}(g)\). These capture whether the intended goal dominates, whether a single non-goal attractor dominates, and whether trajectory mass is dispersed across multiple attractors.

We assign each seed--goal graph to one of four regimes using the following rules:
\[
\begin{aligned}
	\text{Goal-dominant} 
	&\quad \text{if } \beta_g \geq 0.75, \\
	\text{Competitor-dominated} 
	&\quad \text{if } \mathrm{Comp}(g)\geq 0.5 \text{ and } \mathrm{Comp}(g)>\beta_g, \\
	\text{Fragmented} 
	&\quad \text{if } \mathrm{Frag}(g)\geq 0.3,\ \beta_g<0.75,\ \mathrm{Comp}(g)<0.5, \\
	\text{Partial/contested} 
	&\quad \text{otherwise.}
\end{aligned}
\]
The thresholds are deliberately simple and interpretable. They are not optimized to maximize a predictive score. Instead, they encode qualitative distinctions observed repeatedly in the policy maps: goal-basin dominance, dominance by a single non-goal attractor, dispersed failure across several attractors, and intermediate cases in which the goal and competing basins coexist.

Table~\ref{tab:taxonomy_summary} in the main text reports the overall regime summary. The regimes separate sharply by both success and graph structure: goal-dominant cases have high success and large goal basins, competitor-dominated cases have very low success and large non-goal basins, and fragmented cases have intermediate success with trajectory mass dispersed across multiple attractors. Thus, local goal support provides a local screen for likely failure, while the taxonomy explains the global form of residual failure after the full policy graph is examined.

\begin{table}[t]
	\centering
	\small
	\caption{Percentage of seed--goal graphs assigned to each structural regime by condition. Entries are row percentages.}
	\label{tab:app_taxonomy_by_condition}
	\begin{tabular}{lrrrr}
		\toprule
		Condition & Comp.-dom. & Fragmented & Goal-dom. & Partial \\
		\midrule
		$8{\times}8$ TD uniform    & 50.8 &  9.4 & 34.1 & 5.8 \\
		$8{\times}8$ TD edge       & 54.8 &  4.5 & 35.2 & 5.5 \\
		$8{\times}8$ MC uniform    & 46.4 & 42.3 &  5.7 & 5.6 \\
		$12{\times}12$ TD uniform  & 64.9 & 23.4 &  9.6 & 2.1 \\
		$8{\times}8$ BN TD uniform & 36.7 & 49.3 & 11.0 & 3.0 \\
		\bottomrule
	\end{tabular}
\end{table}

Table~\ref{tab:app_taxonomy_by_condition} shows that the distribution of regimes changes systematically across settings. The standard \(8{\times}8\) TD setting produces a substantial number of goal-dominant graphs, but also many competitor-dominated failures. Edge-biased sampling preserves a similar goal-dominant share but does not eliminate competitor-dominated failures. MC learning produces far fewer goal-dominant graphs and many more fragmented cases. The bottleneck environment is especially fragmented, consistent with constrained geometry splitting the state space into separated routing regions. The \(12{\times}12\) setting is dominated by competitor failures, suggesting that larger state spaces make global basin organization harder even when some local support is present.

For a separate descriptive cross-tabulation, we discretize finite-horizon success into five bins:
\[
\begin{aligned}
	\text{Zero} &:\quad \mathrm{Succ}_H(g)=0,\\
	\text{Low} &:\quad 0 < \mathrm{Succ}_H(g) < 0.25,\\
	\text{Partial} &:\quad 0.25 \leq \mathrm{Succ}_H(g) < 0.75,\\
	\text{High} &:\quad 0.75 \leq \mathrm{Succ}_H(g) < 1,\\
	\text{Perfect} &:\quad \mathrm{Succ}_H(g)=1.
\end{aligned}
\]
These bins are used only to interpret the taxonomy; the taxonomy itself is defined from graph quantities, not from success labels.

\begin{table}[t]
	\centering
	\small
	\caption{Success-category composition within each structural regime. Entries are row percentages.}
	\label{tab:app_taxonomy_success_categories}
	\begin{tabular}{lrrrrr}
		\toprule
		Regime & Zero & Low & Partial & High & Perfect \\
		\midrule
		Competitor-dominated & 17.2 & 76.3 &  6.4 &  0.0 &  0.0 \\
		Fragmented           & 12.0 & 47.5 & 40.5 &  0.0 &  0.0 \\
		Goal-dominant        &  0.0 &  0.0 &  3.2 & 36.4 & 60.4 \\
		Partial / contested  &  0.0 &  0.0 &100.0 &  0.0 &  0.0 \\
		\bottomrule
	\end{tabular}
\end{table}

Table~\ref{tab:app_taxonomy_success_categories} provides a sanity check on the taxonomy. Goal-dominant graphs are almost always high- or perfect-success cases. Competitor-dominated graphs are overwhelmingly zero- or low-success cases. Fragmented graphs mostly fall into low or partial success, consistent with dispersed attractor structure. Partial/contested graphs occupy the intermediate partial-success regime. These results support the taxonomy as a compact summary of the qualitative failure modes visible in the policy maps.

\end{document}